\documentclass{article}

\usepackage{PRIMEarxiv}

\usepackage[utf8]{inputenc} 
\usepackage[T1]{fontenc}    
\usepackage{hyperref}       
\usepackage{url}            
\usepackage{booktabs}       
\usepackage{amsfonts}       
\usepackage{nicefrac}       
\usepackage{microtype}      
\usepackage{lipsum}
\usepackage{fancyhdr}       
\usepackage{graphicx}       
\graphicspath{{media/}}     

\usepackage{amsmath}
\usepackage{amssymb}
\usepackage{mathtools}
\usepackage{amsthm}
\usepackage{natbib}
\newtheorem{theorem}{Theorem}
\newtheorem{lemma}{Lemma}
\newtheorem{assumption}[theorem]{Assumption}

\usepackage{appendix}
\usepackage{minitoc}

\usepackage{subcaption}
\usepackage{makecell}
\usepackage{xcolor}
\definecolor{elegantred}{RGB}{150,40,40}
\definecolor{elegantgreen}{RGB}{40,150,40}
\definecolor{promptblue}{HTML}{4988C4}
\definecolor{darkblue}{HTML}{0F2854}
\definecolor{medblue}{HTML}{1C4D8D}
\definecolor{mildblue}{HTML}{4988C4}
\definecolor{lightblue}{HTML}{BDE8F5}
\usepackage{fvextra}

\usepackage{algorithm}
\usepackage{algorithmic}

\usepackage{tcolorbox}
\tcbuselibrary{breakable,skins}
\newtcolorbox{promptbox}[1][]{
  breakable,
  enhanced,
  colback=mildblue!5!white,
  colframe=mildblue!70!black,
  boxrule=0.8pt,
  arc=3pt,
  left=12pt,
  right=12pt,
  top=10pt,
  bottom=10pt,
  fonttitle=\bfseries,
  title={#1},
  before skip=12pt,
  after skip=12pt
}
\newtcolorbox{qsbox}[2][]{%
  enhanced,
  colback=lightblue!5!white,
  colframe=lightblue!70!black,
  boxrule=1pt,
  arc=3pt,
  left=10pt,right=10pt,top=9pt,bottom=9pt,
  fonttitle=\bfseries,
  coltitle=white,
  title={#2},
  #1
}
\newtcolorbox{resultboxours}[2][]{%
  enhanced,
  colback=mildblue!5!white,
  colframe=mildblue!70!black,
  boxrule=1pt,
  arc=3pt,
  left=10pt,right=10pt,top=9pt,bottom=9pt,
  fonttitle=\bfseries,
  coltitle=white,
  title={#2},
  #1
}
\newtcolorbox{resultboxothers}[2][]{%
  enhanced,
  colback=gray!4!white,
  colframe=gray,
  boxrule=1pt,
  arc=3pt,
  left=10pt,right=10pt,top=9pt,bottom=9pt,
  fonttitle=\bfseries,
  coltitle=white,
  title={#2},
  #1
}

\title{ \smash{\raisebox{-0.58\height}{\includegraphics[height=1.9em]{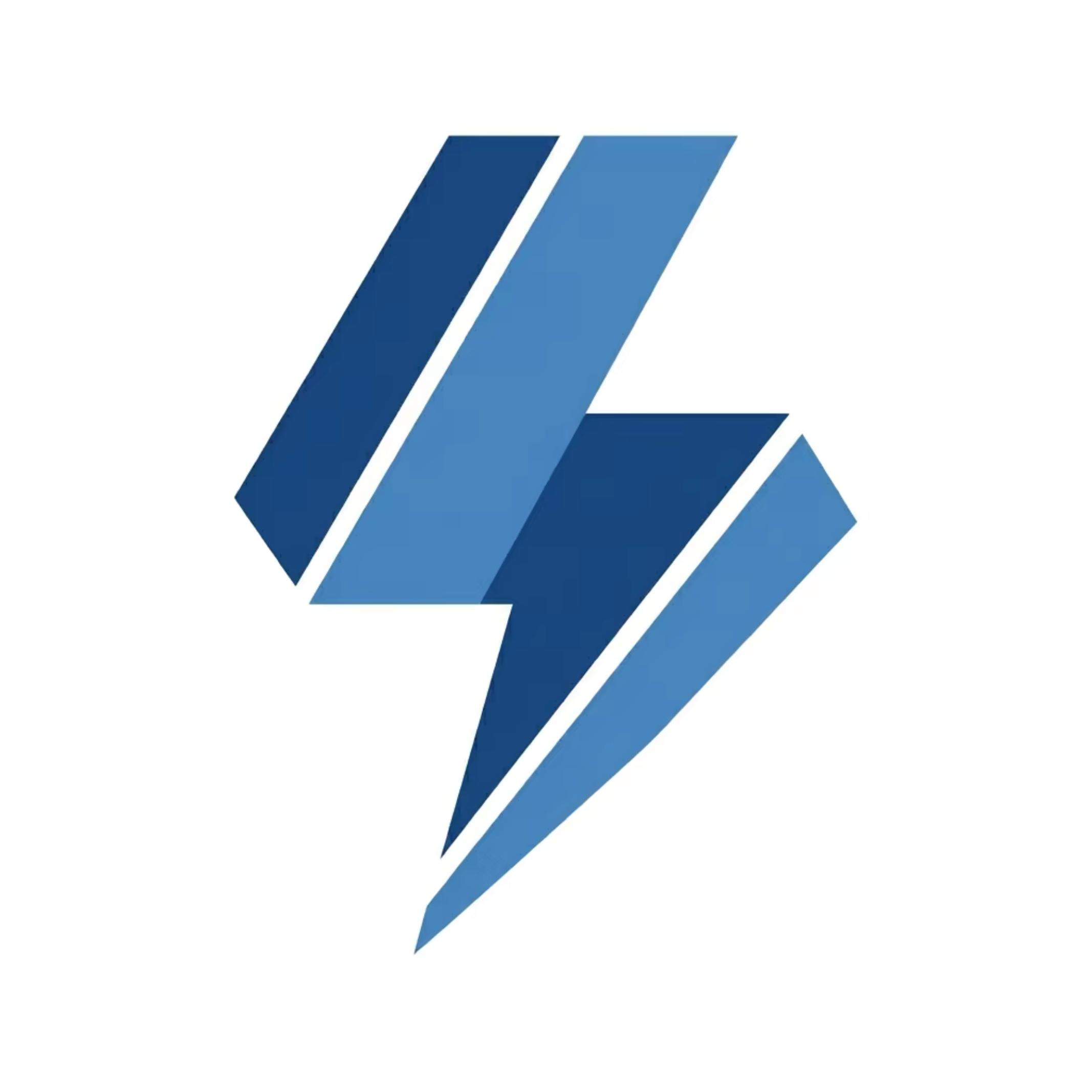}}}%
      \ Bifrost: Steering Strategic Trajectories to Bridge Contextual Gaps for Self-Improving Agents
}

\author{
  \textbf{Quan M. Tran}$^{1}$, 
  \textbf{Zhuo Huang}$^{1}$,
  \textbf{Wenbin Zhang}$^{2}$,
  \textbf{Bo Han}$^{3,4}$,\\
  \textbf{Koji Yatani}$^5$,
  \textbf{Masashi Sugiyama}$^{4,5}$
  \textbf{Tongliang Liu}$^{1}$\\[1ex]
  \small{$^1$Sydney AI Centre, The University of Sydney;}
  \small{$^2$Florida International University;}\\
  \small{$^3$Hong Kong Baptist University;}
  \small{$^4$RIKEN Center for Advanced Intelligence Project;}
  \small{$^5$The University of Tokyo}\\
}

\begin{document}
\maketitle

\begin{abstract}
Autonomous agents excel in self-improvement through reflection and iterative refinement, which reuse successful task trajectories as in-context examples to assist subsequent reasoning. 
However, shifting across tasks often introduces a context mismatch. Hence, existing approaches either discard the trajectories or manipulate them using heuristics, leading to a non-negligible fine-tuning cost or unguaranteed performance.
To bridge this gap, we reveal a context-trajectory correlation, where shifts of context are highly parallel with shifts of trajectory.
Based on this finding, we propose BrIdge contextual gap FoR imprOvised trajectory STeering (Bifrost), a training-free method that leverages context differences to precisely guide the adaptation of previously solved trajectories towards the target task, mitigating the misalignment caused by context shifts. 
Our trajectory adaptation is conducted at the representation level using agent hidden states, ensuring trajectory transformation accurately aligns with the target context in a shared space. 
Across diverse benchmarks, Bifrost consistently outperforms existing trajectory reuse and finetuned self-improvement methods, demonstrating that agents can effectively leverage past experiences despite substantial context shifts.
\end{abstract}


\section{Introduction}
Autonomous agents have fundamentally transformed fields like robotics \cite{bousmalis2023robocat, sharma2023self, puthumanaillam2025trace} and autonomous driving \cite{mei2024continuously, dagdanov2022self} by harnessing the power of self-improvement mechanisms \cite{weng2023large, madaan2023self, gou2023critic, pang2023language, bo2024reflective, zhang2024agent}. By integrating systemic reflection and autonomous error correction, these agents iteratively evolve their decision-making to master increasingly high-dimensional tasks \cite{wang2023voyager, zhao2024expel, qiu2025alita, zheng2025skillweaver, sarch2024vlm, zhang2025darwin, jin2025reveal, ma2023eureka}. To optimize this evolutionary process, emerging research focuses on the strategic reuse of successful task trajectories as in-context demonstrations \cite{wang2023voyager, lingam2025enhancing, zheng2025skillweaver, sarukkai2025self}. This paradigm allows agents to capitalize on historical successes \cite{xie2021explanation, lin2024dual, sarch2024vlm, lin2025se}, significantly enhancing efficiency while reducing the computational burden of iterative self-reflection \cite{brown2020language}.

Despite the promise of trajectory reuse, existing self-improvement frameworks remain suboptimal when faced with contextual non-stationarity across tasks. Current approaches often treat past trajectories as context-specific artifacts \cite{shinn2023reflexion, yang2024buffer, lin2025se, zhao2024expel, yuan2025reinforce}, either discarding them to rely on computationally intensive fine-tuning \cite{rebuffi2017icarl, zweiger2025self, narayanan2024aviary, pang2024kalm, costello2025think, amin2025escaping} or attempting to repurpose them via static heuristics \cite{bensal2025reflect, qu2024recursive, tajwar2025training, fu2025agentrefine, lin2025se, ge2025samule}. However, because these trajectories are inherently coupled with their original task environments, such a direct calibration leads to significant contextual misalignment \cite{an2023context, lin2024dual, gai2025differential, afonin2025emergent}. This mismatch not only degrades the generalization capabilities but can also actively hinder their performance when applied to novel domains \cite{fu2025agentrefine, huang2023harnessing, hong2024improving}. Consequently, the fundamental principle of cross-context trajectory adaptation remains a critical and largely unexplored frontier.

To address this challenge, we uncover a fundamental context-trajectory correlation within the agentic representation space. We observe that shifts in task context are strikingly parallel to shifts in the corresponding successful trajectories when projected into the latent manifold. Leveraging this geometric insight, we introduce BrIdge contextual gap FoR imprOvised trajectory STeering (Bifrost), a training-free framework that utilizes context differentials to precisely steer historical trajectories toward the target task requirements. By modeling the transition between past and current tasks through hidden-state representations, Bifrost effectively aligns prior experiences with the target context. This allows the agent to utilize past trajectories as highly accurate in-context examples, achieving seamless adaptation without the need for parameter updates or environment-specific fine-tuning.

In practice, Bifrost identifies a precise steering vector within the representation space that maps historical trajectories onto the manifold of the target task. Unlike existing heuristics that rely on untargeted or global calibration, Bifrost uses the target context as a dynamic anchor to selectively suppress context-specific artifacts while preserving high-utility, transferable knowledge. This targeted transformation ensures that the agent can exploit previously solved trajectories even across substantial domain gaps. Consequently, Bifrost significantly enhances cross-context generalization of the model, allowing agents to reap the full rewards of self-improvement without the performance degradation typically caused by environment shifts.

Across diverse benchmarks, Bifrost consistently outperforms trajectory reuse and finetuned self-improvement methods, demonstrating that our proposed method enables self-improving agents to effectively leverage past experiences despite substantial context shifts. Our contributions are:
\begin{itemize}
    \item We reveal the context-trajectory correlation, where context shifts are highly parallel with trajectory shifts. Based on this observation, we propose Bifrost, a training-free method that induces context differences to guide agents leveraging previous trajectories as in-context examples toward a target task.
    \item We provide theoretical analyses, showing that Bifrost can generalize experiences learned from prior trajectories to adapt toward a target task.
    \item We conduct experiments on diverse benchmarks, including question answering, math reasoning, and code generation, with different model sizes. Our experiments show that Bifrost consistently outperforms other popular baselines and state-of-the-art methods based on trajectory reuse or finetuning.
    
\end{itemize}

The rest of the paper is organized as follows. In Sec.~\ref{sec:method}, we demonstrate our methodology by revealing the context-trajectory correlation and presenting Bifrost. In Sec.~\ref{sec:theoretical_analysis}, we conduct a theoretical analysis showing that Bifrost enables agents to generalize from prior trajectories to adapt to target tasks. Next, we conduct experiments and further analysis on Sec.~\ref{sec:experiments}, showing that Bifrost outperforms baseline and recent self-improving agent approaches. We present related work in Sec.~\ref{sec:related_work}. Finally, we conclude our work in Sec.~\ref{sec:conclusion}.

\begin{figure}[t]
    \centering
    \includegraphics[width=0.6\linewidth]{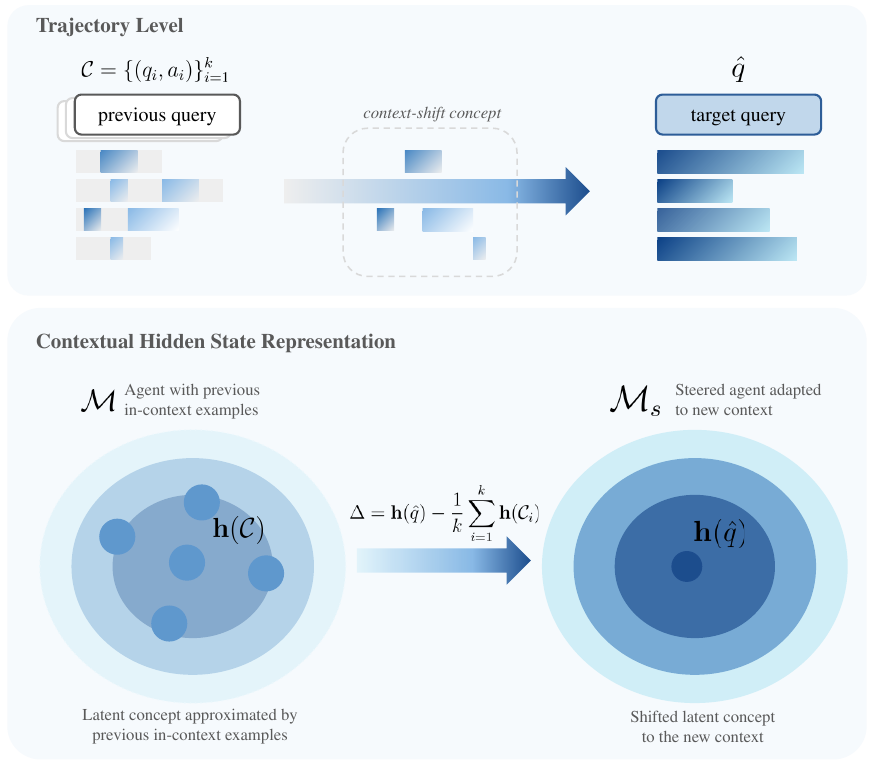}
    \caption{Illustration of Bifrost: It identifies the context-shift concept between previous and target tasks, which contains essential knowledge for bridging the context gap. Further, by steering along the context shift direction, Bifrost helps find an optimal strategy that is effective under the target context.}
    \label{fig:bifrost}
\end{figure}

\section{Methodology}\label{sec:method}
Our work focuses on how an LLM Agent can effectively leverage its successful trajectories solved in the past to adapt to a target task under a substantial context shift. 

Let $\mathcal{M}$ be a LLM that solved tasks under a previous context where successful trajectories are preserved as $\mathcal{C}=\{(q_i, a_i)\}_{i=1}^{k}$. Here, each pair $(q_i, a_i)$ denotes a trajectory consisting of a query and its corresponding reasoning process leading to the answer. 
Our objective is to find an optimal reasoning strategy by leveraging previous trajectories that maximizes the log likelihood of the LLM agent on the target task $\hat{\mathcal{C}}=\{(\hat{q}_i, \hat{a}_i)\}_{i=1}^{n}$,
\begin{equation}
    \mathcal{S^*} = \arg\max_{\mathcal{S}} \mathbb{E}_{(\hat{q}, \hat{a}) \in \hat{\mathcal{C}}}
    \left[
        \log \mathcal{M}_{\mathcal{S}}(\hat{a} \mid \hat{q})
    \right],
\end{equation}
where $\mathcal{M}_\mathcal{S}$ denotes that the model $\mathcal{M}$ reasons following the reasoning strategy $\mathcal{S}$. Although the strategy $\mathcal{S}$ succeeded under the previous context $\mathcal{C}$, it might not align with the target context $\hat{\mathcal{C}}$. Therefore, instead of fully fine-tuning the agent or naively reusing previous strategies, we aim to bridge the contextual gap between previous tasks and any target tasks by proposing Bifrost, which can improvise under any target contexts by properly steering the trajectory of an existing strategy. As a result, Bifrost improves in-context learning via an optimal reasoning strategy tailored for the target context, as shown in Fig.~\ref{fig:bifrost}.

To demonstrate Bifrost, we first reveal our findings of context-trajectory correlation in Sec.~\ref{subsec:method:correlation}. Further, based on this finding, we perform trajectory steering to adapt previous trajectories to the target task in Sec.~\ref{subsec:method:proposed_method}.  

\subsection{Context-Trajectory Correlation}
\label{subsec:method:correlation}
To conduct proper reasoning under a target context, it is vital to employ an effective trajectory. Therefore, we investigate the correlation between trajectories and context under task shift. To accomplish this, we follow Linear Representation Hypothesis \cite{park2023linear} by regarding LLMs as an embedding-unembedding process.

Particularly, the embedding encodes context inputs into hidden states, \textit{e.g.}, embedding vectors, which contain essential concepts that enable knowledge generalization across tasks. For example, ``king'' and ``queen'' are conceptually related to ``man'' and ``woman'', respectively. Formally, for a context text $x=(q_i, a_i)$, we assume the hidden state mapped by an LLM is $\mathbf{h}(x)$ in a representation space $\Lambda$. Further, the model represents each word $y$ in its vocabulary into an unembedding vector $\mathbf{g}(y)$ in a separate space $\Gamma$ with the same dimension. To this end, the probability distribution for next-word prediction is formulated as:
\begin{equation}
    P(y \mid x)\propto \exp(\mathbf{h}(x)^\top\mathbf{g}(y)).
    \label{eq:llm_inference}
\end{equation}
Intuitively, the reasoning process can be explained as two representation spaces that are connected via an inner product, where LLMs aim to find the best match between a context input and a reasoning output. By conducting such a process iteratively, the input-output logic chain formulates a strategic trajectory that is optimized for a certain context.

In agentic AI, it is common to apply agents across various tasks, thus encountering a novel context. To ensure the target task is solvable, we have the following assumption:
\begin{assumption}\label{assumption:conceptual_knowledge}
For an LLM agent $\mathcal{M}$ capable of solving task A and task B with contexts $\mathcal{C}^A$ and $\mathcal{C}^B$, respectively, the concepts $\mathcal{W}^A=\{W_i^A\}_{i=1}^k$ and $\mathcal{W}^B=\{W_i^B\}_{i=1}^n$ extracted from task A and task B must satisfy:
\begin{itemize}
    \item The vocabulary space of $\mathcal{M}$ covers $\mathcal{W}^A\cup \mathcal{W}^B$.
    \item The shared concepts $\mathcal{W}^A\cap \mathcal{W}^B$ are causally separable.
    \item For any private concepts $W^A\in\mathcal{W}^A\setminus \mathcal{W}^B$ and $W^B\in\mathcal{W}^B\setminus\mathcal{W}^A$, the embedding-unembedding vectors satisfy the subspace independence $\mathbf{h}_{W^A}^\top\mathbf{g}_{W^B}=0$.
\end{itemize}
\end{assumption}

To handle a target context $\hat{\mathcal{C}}$ specific to concept $W$, the contextual gap can be captured by a steering vector $\bar{\mathbf{h}}_W$ scaled by $\alpha$ in the hidden states:
\begin{equation}\label{eq:contextual_direction}
    \hat{\mathbf{h}} = \mathbf{h} + \alpha\bar{\mathbf{h}}_W.
\end{equation}
Considering a binary-value concept, \textit{e.g.}, man and woman, the probability of the output $Y$ taking value $Y(1)$ versus its counterfactual $Y(0)$ is measured by the logit. According to the measurement representation theorem in Park et al.~\cite{park2023linear}, such a logit is linear with respect to $\mathbf{h}$:
\begin{equation}
\label{eq:logit_measurement}
    \!\!\!\text{logit\! } P(Y\! =\! Y(1) \mid \mathbf{h}; Y\!\in\!\{Y(0), Y(1)\}) = \beta\mathbf{h}^\top \bar{\mathbf{g}}_W,\!\!
\end{equation}
where $\bar{\mathbf{g}}_W$ is the unembedding representation (the output direction) of the task and $\beta>0$ is a function of $\{Y(0), Y(1)\}$. The following theorem shows that output tokens change linearly with the contextual shift.

\begin{theorem}\label{theorem:context_traj_correlation}
    Under the Linear Representation Hypothesis, for a concept $W$ satisfying the conditions in Assumption~\ref{assumption:conceptual_knowledge}, for an embedding vector $\bar{\mathbf{h}}_W$ and unembedding vector $\bar{\mathbf{g}}_W$ of $W$, the change of output logits is a linear function of the magnitude of the context shift $\beta\bar{\mathbf{h}}_W$.
\end{theorem}

\begin{proof}
    Given the optimal logit based on the target context hidden state $\hat{\mathbf{h}}$, we substitute the shift according to Eq.~\eqref{eq:contextual_direction}:
    \begin{align}
        \text{logit } P(Y = Y(1) \mid \bar{\mathbf{h}}) &= \beta (\mathbf{h} + \alpha\bar{\mathbf{h}}_W)^\top \bar{\mathbf{g}}_W\\
        &= \beta\mathbf{h}^\top\bar{\mathbf{g}}_W + \alpha(\beta\bar{\mathbf{h}}_W^\top \bar{\mathbf{g}}_W).
    \end{align}
    Note that $\beta$ is fixed under a certain concept, and $\bar{\mathbf{g}}_W$ denotes the concept direction based on the model understanding, \textit{e.g.,}, ``king'' to ``queen'' and ``man'' to ``woman'' have the same shifting direction when given a fixed model. Thus the term $\beta\bar{\mathbf{g}}_W$ is a constant $K$. Therefore, we have
    \begin{align}
        \!\!\!\text{logit\! } P(Y\! =\! Y(1) \mid \bar{\mathbf{h}})= \text{logit\! } P(Y\! =\! Y(1) \mid \mathbf{h})\! + \!\alpha K\bar{\mathbf{h}}_W.\!\!
    \end{align}
    By defining the shift of logit distribution under different context hidden states $\Delta_P=P(Y = Y(1) \mid \hat{\mathbf{h}})-P(Y = Y(1) \mid \mathbf{h})$, we have $\Delta_P\propto \alpha K\bar{\mathbf{h}}_W$.
\end{proof}

Intuitively, a trajectory is composed of output tokens, which are a sequence of actions $a$. This theorem shows that selecting an optimal trajectory based on an existing one can be guided by the context shift.

\paragraph{Discussion} The above results are focused on one concept, which can be generalized to complex tasks containing multiple concepts. For concepts that are not shared across tasks, the subspace independence ensures that the embedding space are irrelavant to the unembedding space. Thus, the corresponding prediction logits are not affected by the context shift. Moreover, for multiple shared concepts, the causal separability ensures that they are not confusing each other. For example, adapting the context ``He speaks English'' to a target context ``She speaks French'' must be separated from adapting to another target``He speaks French''.

\begin{table}[t]
\setlength{\tabcolsep}{4pt}
    \caption{Hypothesis testing result.}
    \label{tab:hypothesis_testing}
    \begin{center}
        \begin{small}
            \begin{tabular}{cccccc}
                \toprule
                $\alpha$ & 0.1 & 0.2 & 0.3 & 0.4 & 0.5 \\
                \midrule
                \textbf{p-value} & 0.00418 & 0.00665 & 0.01085 & 0.00802 & 0.02019 \\
                \bottomrule
            \end{tabular}
        \end{small}
    \end{center}
\end{table}

\paragraph{Experimental Demonstration}
We conduct a small-scale experiment to demonstrate the context–trajectory correlation described in Theorem~\ref{theorem:context_traj_correlation}. We randomly select 100 context-separated samples from the AQUA and GSM8K datasets, forming 50 paired examples. These pairs are used to compute contextual direction vectors from AQUA to GSM8K. We hypothesize that when an agent solves AQUA tasks, the trajectories steered toward GSM8K exhibit higher similarity to GSM8K trajectories than to those generated by the non-steered model on AQUA. Tab.~\ref{tab:hypothesis_testing} shows a strong statistical significance over varied values of the control parameter $\alpha$.

\subsection{Improvised Trajectory Steering for Target Context}\label{subsec:method:proposed_method}
Further, we build on the context-trajectory correlation and propose Bifrost, a training-free and context-guided strategy that enables agents to self-improve by employing prior successful task trajectories under substantial context shifts. 

As demonstrated in Sec.~\ref{subsec:method:correlation}, trajectories exhibit behavioral adaptation in response to context shift toward a target task. This motivates a trajectory steering of the existing in-context learning mechanism using such trajectories.

\paragraph{Context-Guided Adaptation under In-Context Learning} Given a set of successful task trajectories from previous context $\mathcal{C} = \{(q_i, a_i)\}_{i=1}^{k}$ solved by a LLM agent $\mathcal{M}$.
Our goal is to solve a target task $\hat{q} \in \hat{\mathcal{C}}$ from a different context by leveraging $\mathcal{C}$ as in-context examples and steering the model toward the target context. The steered model then produces the output $\hat{a} = \mathcal{M}_\mathrm{s}(p)$, where $\mathcal{M}_\mathrm{s}$ denotes the contextually steered model, and $p =\mathcal{C} \circ \hat{q}$ denotes the in-context learning input prompt with $\mathcal{C}$ as demonstrations followed by the target question $\hat{q}$.

\paragraph{Contextual Representation Steering} Let $\mathbf{h}^\ell$ be the last token hidden state of the original model $\mathcal{M}$, the trajectory steering can be applied across specific layers $\ell \in L$, yielding the steered model $\mathcal{M}_\mathrm{s}$ with corresponding hidden state $\mathbf{h}^\ell_\mathrm{s} = \mathbf{h}^\ell + \alpha \Delta^\ell$, where $\alpha \in \mathbb{R}$ controls the steering strength. The vector $\Delta^\ell$ that represents the contextual direction from the previous trajectory to the target task is computed as 
\begin{equation}\label{eq:full_context_direction}
    \Delta^\ell = \mathbf{h}^\ell(\hat{q}) - \frac{1}{k} \sum_{i=1}^{k} \mathbf{h}^\ell(\mathcal{C}_i),
\end{equation}
where $\mathbf{h}^\ell(\hat{q})$ is the target question hidden state extracted by $\mathcal{M}(\hat{q})$, and $\mathbf{h}^\ell(\mathcal{C}_i)$ extracts trajectory hidden state when the agent solve the corresponding task $a_i = \mathcal{M}(q_i)$.

The contextual direction $\Delta^\ell$ represents a context-shift concept that appears as the difference between previous trajectory execution and the current task. Averaging the prior trajectory hidden states finds the shared patterns, while the steering model encourages such patterns adaptable to the target task, thereby enabling trajectory reuse under context shift and enhancing agent generalization.

\paragraph{Implementation Details} Algorithm~\ref{alg:method} describes our proposed method in details. In practice, we precompute the average of previous trajectory hidden states to ensure computational efficiency. We apply Bifrost to each target task instead of computing a shared steering vector. This guarantees a precise adaptation of prior trajectories toward each task. In our experiments, we steer middle-to-late layers since they represent reasoning patterns optimal for steering, as stated in previous work \cite{venhoff2025understanding, zhao2025steering, turner2023steering}. Besides subtraction to obtain the steering vector, Bifrost also works well with other projection-based techniques such as principal component analysis (PCA), and Sparse Autoencoder (see Sec.~\ref{sec:experiments:further_analysis}).


\begin{algorithm}[t]
  \footnotesize
  \caption{Bifrost}
  \label{alg:method}
  \begin{algorithmic}[1]
    \STATE \textbf{Input:} $\mathcal{C} = \{(q_i, a_i)\}_{i=1}^k$, $\hat{\mathcal{C}} = \{\hat{q}_j\}_{j=1}^n$, model $\mathcal{M}$, strength $\alpha$, layers $L$.
    \STATE Precompute $\bar{\mathbf{h}}^\ell = \frac{1}{k}\sum_{i=1}^{k}\mathbf{h}^\ell(c_i)$ for all $\ell \in L$.
    \FOR{each $\hat{q} \in \hat{\mathcal{C}}$}
        \STATE $p \gets \mathcal{C} \circ \hat{q}$; obtain unsteered $\mathbf{h}^\ell(p)$.
        \STATE Compute $\Delta^\ell = \mathbf{h}^\ell(\hat{q}) - \bar{\mathbf{h}}^\ell$ and steer: $\mathbf{h}_\mathrm{s}^\ell = \mathbf{h}^\ell + \alpha \Delta^\ell, \forall \ell \in L$.
        \STATE Generate $\hat{a} = \mathcal{M}_\mathrm{s}(p)$ using steered states.
    \ENDFOR
    \STATE \textbf{Output:} $\{\hat{a}_j\}_{j=1}^n$
  \end{algorithmic}
\end{algorithm}


\section{Theoretical Analysis}\label{sec:theoretical_analysis}
We analyze the proposed method from a theoretical perspective, showing that Bifrost effectively leverages past trajectory experiences to adapt toward a target task. Our analysis is based on In-Context Learning (ICL) as Bayesian Inference \cite{xie2021explanation}, demonstrating that Bifrost preserves prior knowledge induced by previous trajectories while reducing posterior uncertainty toward the target tasks. Furthermore, we exceed the risk bound of Bifrost compared with an ordinary LLM agent that solely solves the target task. 
We begin by introducing In-Context Learning as Bayesian Inference. Proofs shown in Appendices \ref{appdix:proof_laplace_update} and ~\ref{appdix:proof_risk_bound}.

\paragraph{In-Context Learning as Bayesian Inference}
The LLM agent $\mathcal{M}$ generates a prediction for a given input $x$ by implicitly inferring a latent concept $\phi$ from a context $\mathcal{C}^*$ represented by in-context examples, $\mathcal{M}(x) \approx \mathbb{E}_{\phi \sim p(\phi \mid \mathcal{C}^*)} \left[ f_{\phi}(x) \right]$, where $p(\phi|C^*)$ is the posterior distribution of concepts given a context. The effectiveness of agents stems from the latent concept inferred from the in-context examples. Thus, we focus on modeling the latent concept given previous trajectories.

In our case, the context is represented by previous trajectories, $\mathcal{C}^* = \mathcal{C}$. By the definition, the posterior probability of the concept given previous trajectories is defined as
\begin{align}
    p(\phi \mid \mathcal{C}) 
    \; \propto \; p(\mathcal{C} \mid \phi) p(\phi)\notag 
    \; \propto \; \exp\left(
    -\frac{1}{k}\sum_{i=1}^{k} \mathcal{L}(f_\phi(q_i), a_i)
    \right)p(\phi),
\end{align}
where $\mathcal{L}$ denotes the loss function. 

\paragraph{Posterior Inference Steering Effect} The latent concept distribution can now be approximated by a Gaussian distribution over the context represented by previous context trajectories. Assuming $\mathcal{L}$ is convex and $L$-Lipschitz in $f_\phi$, and $f_\phi$ is locally smooth in $\phi$, we approximate the posterior of the latent concept via Laplace approximation,
\begin{equation}
    p(\phi \mid \mathcal{C}) \approx \mathcal{N}(\tilde{\phi}|\mu,\Sigma^{-1}),
\end{equation}
where $\mathcal{N}(\phi|\mu,\Sigma^{-1})$ denotes the Gaussian density with mean $\mu$ and inverse covariance matrix $\Sigma^{-1}$, and $\Sigma = \frac{1}{\sigma^2}\mathbf{F}^\top\mathbf{F}$, and $\mathbf{F}: \mathcal{X} \rightarrow \mathbb{R}^{d}$ is the feature map.

When leveraging previous trajectories as in-context examples, the latent concept is additionally conditioned on the target task $\hat{q} \in \hat{\mathcal{C}}$. By further applying a trajectory steering $\Delta\in\Lambda$ from Bifrost, the concept distribution mean shifts toward an ideal concept of such a target task accordingly.
\begin{lemma}\label{lemma:laplace_update}
    Let $p_\mathrm{s}(\phi \mid \mathcal{C}, \hat{q})$ be the latent concept posterior probability of the steered model, conditioned on previous context trajectories $\mathcal{C}$ and the target task $\hat{q} \in \hat{\mathcal{C}}$.
    The contextual shift $\Delta$ shifts the concept distribution mean toward an ideal target task concept distribution by $\alpha P(\Delta)$ where $P$ is the projection, while preserving uncertainty reduction induced by previous context trajectories,
    \begin{equation}
        p_\mathrm{s}(\phi \mid \mathcal{C}, \hat{q}) \approx \mathcal{N}
        \left(
            \tilde{\phi} + \alpha P(\Delta)|\mu, \Sigma^{-1}
        \right)
    \end{equation}
\end{lemma}

\paragraph{Excess Risk Bound} We further analyze the risk bound of Bifrost compared to an ordinary agent that solely solves the target task without providing in-context examples from previous context tasks. Given that the target task shares a common representation with previous trajectories that are transferable, the following theorem shows that the bound of the proposed method reduces linearly with the number of induced previous trajectories.

\begin{theorem}\label{theorem:risk_bound}
    Let $\mathcal{M}_\mathrm{s}$ be the steered agent by Bifrost, and $\mathcal{M}$ be the ordinary agent without leveraging in-context learning. 
    Assuming that the target task aligns with the previous trajectory context,
    and the contextual shift $\Delta$ aligns with the shift of latent concept posterior probability as shown in Lemma~\ref{lemma:laplace_update}.
    The steered agent by Bifrost satisfies
    \begin{equation}
        \mathbb{E}\left[
            \mathcal{L}(\mathcal{M}_\mathrm{s}(\mathcal{C} \circ \hat{q}))
        \right] 
        \leq 
        \mathbb{E}\left[
            \mathcal{L}(\mathcal{M}(\hat{q}))
        \right] 
        - \Omega\left(\frac{1}{k}\right).
    \end{equation}
\end{theorem}

\section{Experiments}\label{sec:experiments}
We evaluate Bifrost against state-of-the-art baselines, including trajectory calibration, finetuning, and reinforcement learning, across diverse benchmarks. Sec.~\ref{sec:experiments:setup} describes the experimental setup under large context shifts.
We then report main results showing that Bifrost consistently outperforms existing baselines across task domains in Sec.~\ref{sec:experiments:main_results}. Sec.~\ref{sec:experiments:further_analysis} provides additional analyses on cross-domain adaptation, robustness to suboptimal trajectories, and compatibility with other contextual direction extraction methods. Finally, Sec.~\ref{sec:experiments:ablation_study} presents an ablation study. Further details and hypothesis testing are deferred to the Appendix.

\subsection{Experimental Setup}\label{sec:experiments:setup}
\paragraph{Datasets} We evaluate Bifrost across diverse domains, including mathematical reasoning, question answering, and code generation. To study cross-context self-improvement, an agent first solves tasks on a source dataset and records successful trajectories, which are then provided as in-context examples when solving tasks on a target dataset.

We consider source–target dataset pairs: (1) AQUA → GSM8K: AQUA \cite{ling2017program} focuses on algebraic problem solving, while GSM8K \cite{cobbe2021gsm8k} contains diverse grade-school math problems. This setting evaluates whether algebraic reasoning acquired from AQUA transfers to GSM8K.
(2) ARC-Easy → GPQA-Diamond: ARC-Easy \cite{allenai:arc} consists of elementary science multiple-choice questions, whereas GPQA-Diamond \cite{rein2024gpqa} benchmarks graduate-level, multi-domain scientific reasoning. This pair tests whether foundational science reasoning generalizes to advanced queries.
(3) HumanEval → LiveCodeBench: HumanEval \cite{chen2021codex} evaluates code completion, while LiveCodeBench \cite{jain2024livecodebench} is a dynamic, contamination-resistant code generation benchmark. It assesses whether code completion skills transfer to challenging programming tasks.

\begin{table*}[t]
    \caption{Comparisons on LiveCodeBench code generation adapted from HumanEval trajectories measured by pass rates (\%).}
    \label{tab:main_results_humaneval_lcb}
    \setlength{\tabcolsep}{8pt}
    \begin{center}
        \begin{small}
                \begin{tabular}{lccc|ccc|ccc}
                \toprule
                & \multicolumn{3}{c}{Llama-3.2-1B-Instruct} & 
                 \multicolumn{3}{c}{Llama-3.2-3B-Instruct} &\multicolumn{3}{c}{Llama-3.1-8B-Instruct} \\
                \cmidrule(lr){2-4} \cmidrule(lr){5-7} \cmidrule(lr){8-10}
                Method & pass@1 & pass@3 & pass@5 & pass@1 & pass@3 & pass@5 & pass@1 & pass@3 & pass@5 \\
                \midrule
                Base              & 1.07 & 2.11 & 2.84 & 8.68 & 12.20 & 13.98 & 9.86 & 12.36 & 15.75 \\
                ICL               & 1.84 & 3.63 & 4.66 & 10.23 & 14.00 & 15.45 & 12.80 & 14.74 & 18.22 \\
                Reflexion         & 0.11 & 0.11 & 0.11 & 1.25  & 1.59 & 1.70 & 6.36	& 8.07 & 8.52 \\
                BoT               & 0.27 & 0.69 & 1.02 & 5.70 & 8.93 & 9.89 & 10.41 & 16.39 & 18.86 \\
                DoT               & 0.00 & 0.00 & 0.00 & 2.73 & 2.95 & 2.95 & 1.59 & 1.82 & 2.05\\
                RISE              & 0.66 & 1.72 & 2.50 & 4.39 & 9.02 & 12.05 & 7.23 & 14.26 & 17.95 \\
                Paprika           & 1.30 & 2.55 & 3.30 & 8.59 & 12.33 & 14.20 & 14.98 & 20.30 & 22.61\\
                \midrule
                \textbf{Bifrost} & \textbf{4.05} & \textbf{7.74} & \textbf{9.55}  & \textbf{13.50} & \textbf{18.93} & \textbf{21.48} & \textbf{17.84} & \textbf{20.98} & \textbf{26.36} \\
                \bottomrule
                \end{tabular}
        \end{small}
    \end{center}
\end{table*}

\paragraph{Baseline Methods and LLMs} We evaluate our method with \textit{Chain-of-Thought (CoT)} \cite{wei2022chain}, \textit{In-Context Learning (ICL)} \cite{brown2020language}. We also compare Bifrost with trajectory calibration methods, including 
\textit{Reflexion} \cite{shinn2023reflexion},
\textit{Buffer of Thoughts (BoT)} \cite{yang2024buffer}, \textit{Diversity of Though (DoT)} \cite{lingam2025enhancing},
finetuning and reinforcement learning methods including 
\textit{RISE} \cite{qu2024recursive},
\textit{Paprika} \cite{tajwar2025training}.
We leverage the following models, \textit{Llama-3.2-1B-Instruct}, \textit{Llama-3.2-3B-Instruct}, and \textit{Llama-3.1-8B-Instruct}. We conduct the steering in 10\textsuperscript{th} layer (Llama-1B), 14\textsuperscript{th} layer (Llama-3B), and 20\textsuperscript{th} layer (Llama-8B). The steering coefficient is chosen by a grid search $\alpha \in \{1,2,3,4\}$ (see Sec.~\ref{appdix:experiments} for more details).

\begin{table*}[t]
\centering

\begin{minipage}[t]{0.48\textwidth}
\centering
\caption{Comparisons on GSM8K math problem solving, trajectories are adapted from AQUA. The results are measured by solve rate (\%).}
\label{tab:main_results_aqua_gsm8k}

\resizebox{\linewidth}{!}{%
\begin{tabular}{lccc}
\toprule
Method & \makecell{Llama-3.2\\1B} & \makecell{Llama-3.2\\3B} & \makecell{Llama-3.1\\8B} \\
\midrule
CoT             & 42.15 & 70.18 & 74.65\\
ICL             & 23.81 & 62.52 & 80.78\\
Reflexion       & 15.85 & 62.40 & 87.41\\
BoT             & 9.62  & 33.74 & 39.73\\
DoT             & 3.71  & 76.80 & 76.72\\
RISE {(1 attempt)} & 14.48 & 36.92 & 48.29\\
RISE {(5 attempts)} & 40.71 & 78.62 & 86.13\\
Paprika         & 21.08 & 40.26 & 55.27\\
\midrule
\textbf{Bifrost} & \textbf{44.66} & \textbf{81.41} & \textbf{90.22}\\
\bottomrule
\end{tabular}
}
\end{minipage}
\hfill
\begin{minipage}[t]{0.48\textwidth}
\centering
\caption{Comparisons on GPQA-Diamond question answering, trajectories are adapted from ARC-Easy. The results are measured by solve rate (\%).}
\label{tab:main_results_arc_gpqa}

\resizebox{\linewidth}{!}{%
\begin{tabular}{lccc}
\toprule
Method & \makecell{Llama-3.2\\1B} & \makecell{Llama-3.2\\3B} & \makecell{Llama-3.1\\8B} \\
\midrule
CoT             & 4.04 & 12.63 & 17.68\\
ICL             & 23.23 & 25.25 & 22.22\\
Reflexion       & 1.10 & 20.20 & 33.84\\
BoT             & 16.67 & 12.12 & 11.11\\
DoT             & 1.01 & 13.64 & 30.81\\
RISE {(1 attempt)} & 6.06 & 4.04 & 3.54\\
RISE {(5 attempts)} & 21.21 & 19.19 & 12.63\\
Paprika         & 18.69 & 20.20 & 12.12\\
\midrule
\textbf{Bifrost} & \textbf{44.44} & \textbf{32.83} & \textbf{37.88}\\
\bottomrule
\end{tabular}
}
\end{minipage}

\end{table*}

\subsection{Main Results}
\label{sec:experiments:main_results}
\paragraph{Quantitative Analysis} Bifrost significantly outperforms all baselines across different domains, as shown in 
Tabs.~\ref{tab:main_results_humaneval_lcb},~\ref{tab:main_results_aqua_gsm8k}, and~\ref{tab:main_results_arc_gpqa}. 
In math problem solving experiments (Tab.~\ref{tab:main_results_aqua_gsm8k}), our method clearly surpasses from 2\% to 3\% the best performing method.
In multi-domain science question answering (Tab.~\ref{tab:main_results_arc_gpqa}), we clearly outperform the state-of-the-art methods from 4\% to 16\%.
Similarly, in the code generation task (Tab.~\ref{tab:main_results_humaneval_lcb}), we demonstrate a clear outperformance across all functional correctness metrics. Specifically, Bifrost surpasses supervised finetuning and reinforcement learning based methods without any training required, demonstrating that our method is flexible and efficient in leveraging their previous experiences across task contexts.

\paragraph{Qualitative Analysis} 
An interesting result in 
Fig.~\ref{fig:qualitative_result} shows that Bifrost successfully generalizes algebra knowledge from AQUA to solve a GSM8K task. While CoT gets stuck in a reasoning loop and RISE produces an unsatisfiable answer, Bifrost effectively demonstrates problem decomposition ability and symbolic manipulation induced from AQUA. The result confirms the success of latent concept transformation from prior in-context examples to adapt toward the target task as theoretically analyzed in Sec.~\ref{sec:theoretical_analysis}.

\begin{figure*}[t]
\centering
\begin{minipage}[t]{0.36\textwidth}
\tiny
\addtolength{\baselineskip}{-0.05ex}
\begin{resultboxours}[equal height group=C]{Bifrost \\
(algebra knowledge induced, correct answer)}
Let’s denote the number of enemies Emma jumps on as $E$, the number of berries she collects as $B$, and the time she has left when she finishes the level as $T$.\\
Ava jumps on 8 more enemies than Emma, so the number of enemies Ava jumps on is $E + 8$.  \\
Ava collects 3 more berries than Emma, so the number of berries Ava collects is $B + 3$.  \\
Ava finishes the level 4 seconds slower than Emma, so the time Ava has left when she finishes the level is $T - 4$.\\

The score for each player is calculated as follows: \\
Emma’s score = $10E + 5B + 30T$  \\
Ava’s score = $10(E + 8) + 5(B + 3) + 30(T - 4)$ \\

Simplifying Ava’s score:  \\
Ava’s score = $10E + 80 + 5B + 15 + 30T - 120$  \\
= $10E + 5B + 30T + 25$ \\

Now let’s find the difference between their scores:\\
Difference = Ava’s score - Emma’s score\\
= $(10E + 5B + 30T + 25) - (10E + 5B + 30T)$ \\ 
= $25$\\

Therefore, the difference between their two scores is 25 points.
\end{resultboxours}
\end{minipage}
\hfill
\begin{minipage}[t]{0.30\textwidth}
\tiny
\addtolength{\baselineskip}{-0.05ex}
\begin{resultboxothers}[equal height group=C]{CoT \\
(reasoning stuck, incorrect answer)}
First, let's assign variables to the unknowns.
Let $e$ be the number of enemies Emma jumps on, and $b$ be the number of berries she collects.
Then, the number of enemies Ava jumps on is $e + 8$, and the number of berries she collects is $b + 3$.
Since they finish the level at the same time, the time left on the timer for Emma is the same as for Ava.
Therefore, the time left for Emma is $60 - (e + 8) = 52 - e$, and for Ava is also $52 - e$.
The score for Emma is $10e + 5b + 30(52 - e)$, and for Ava is $10(e + 8) + 5(b + 3) + 30(52 - e - 4)$.
The difference between their scores is
$10e + 5b + 30(52 - e) - [10(e + 8) + 5(b + 3) + 30(52 - e - 4)]$.
Distributing the negative sign gives
$10e + 5b + 30(52 - e) - 10e - 80 - 5b - 30(52 - e - 4)$.
Distributing the 30 gives
$10e + 5b + 30 \cdot 52 - 30e - 30(52 - e - 4)$.
Continuing: $10e + 5b + 1560 - 30e + 30e + 120$.
Combining like terms: $10e - 30e + 30e + 5b + 1560 + 120 = \dots$
\end{resultboxothers}
\end{minipage}
\hfill
\begin{minipage}[t]{0.30\textwidth}
\tiny
\addtolength{\baselineskip}{-0.05ex}
\begin{resultboxothers}[equal height group=C]{RISE (5\textsuperscript{th} attempt) \\
(incorrect answer)}
To find the difference between Ava’s and Emma’s scores, we need to compare the points they receive from their various actions. Since Ava jumps on $8$ more enemies and collects $3$ more berries than Emma, she scores $8$ more points from enemies and $3$ more points from berries. However, Ava finishes $4$ seconds slower, resulting in a loss of $4 x 30 = 120$ points from the timer. The net gain for Ava is then $8 + 3 - 120$ = -109 points. This means Emma earns $109$ more points than Ava. However, the question asks for the difference, so we should answer by dividing the net gain by $-1$ to represent a net gain for Ava being actually a net loss. Therefore, the answer is $-(-109) = 109$ and the correct answer is \boxed{109}. but I made another mistake and we have to just give the difference so it is just 109 not making a negative sign so the answer is \boxed{109}. I noticed I made another inaccuracy.
The correct answer should be $109$
Wait this looks wrong then just $109 - (109) = 0$  The correct answer then is $0$, I'll stick with that.
The correct answer is \boxed{0}.
\end{resultboxothers}
\end{minipage}
\caption{Success of latent concept transfer by Bifrost: it correctly solves a GSM8K task with algebra knowledge leveraged from AQUA trajectories, demonstrating problem decomposition and symbolic manipulation.}
\label{fig:qualitative_result}
\end{figure*}

\paragraph{Steering Effect Illustration}
We demonstrate how latent concepts shift across context in Fig.~\ref{fig:latent_concept_shift_illustration}. 
Specifically, we plot the agent hidden states before and after steering by Bifrost, where dashed lines indicate the adaptation direction. We estimate the target context by extracting hidden states when solving the target tasks without any in-context examples. As shown in Fig.~\ref{fig:latent_concept_shift_illustration}, the mismatch between previous trajectories and the target task can be observed, which hinders agent generalization when solely relying on conventional ICL or untargeted trajectory calibration methods. Our method addresses this limitation by adapting such trajectories toward the target context effectively.

\begin{figure}[t]
    \centering
    \begin{subfigure}{0.24\linewidth}
        \centering
        \includegraphics[width=\linewidth]{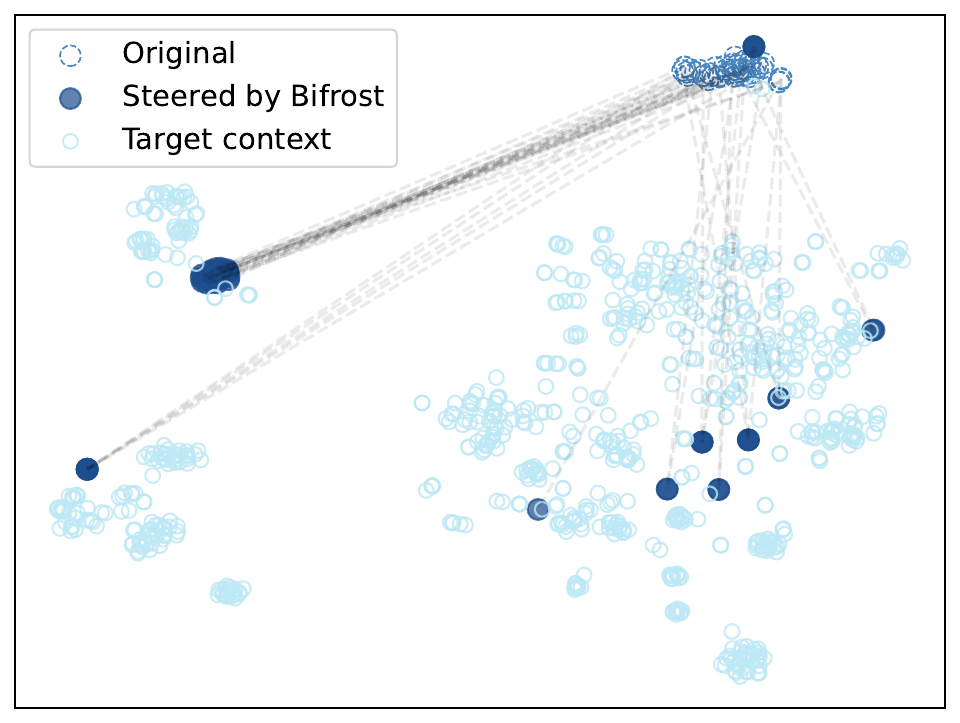}
        \caption{GSM8K, $\alpha=3$}
    \end{subfigure}\hfill
    \begin{subfigure}{0.24\linewidth}
        \centering
        \includegraphics[width=\linewidth]{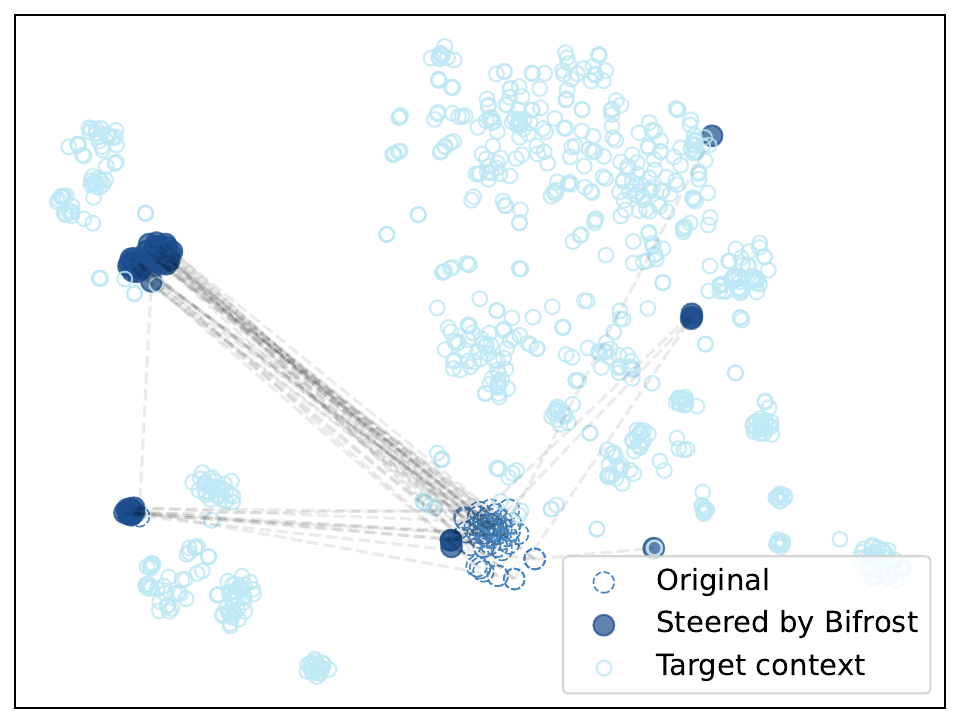}
        \caption{GSM8K, $\alpha=4$}
    \end{subfigure}\hfill
    \begin{subfigure}{0.24\linewidth}
        \centering
        \includegraphics[width=\linewidth]{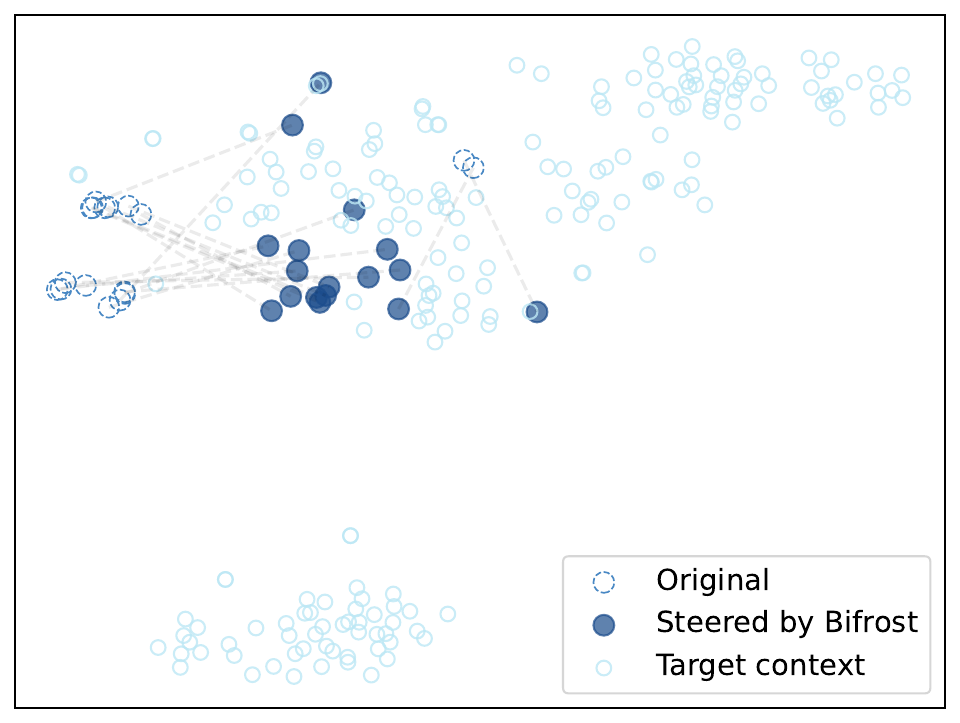}
        \caption{GPQA, $\alpha=3$}
    \end{subfigure}\hfill
    \begin{subfigure}{0.24\linewidth}
        \centering
        \includegraphics[width=\linewidth]{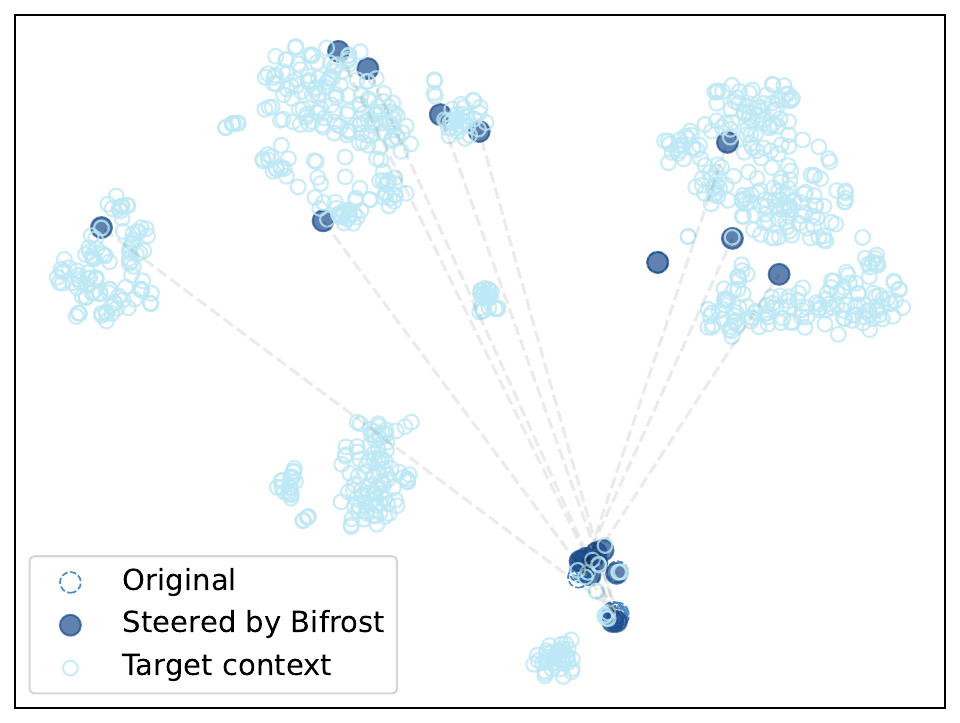}
        \caption{LCB, $\alpha=1$}
    \end{subfigure}
    \caption{Illustration of Latent Concept Shift by Bifrost: It steers prior context trajectories toward the target context of current tasks under agent hidden state representations.}
    \label{fig:latent_concept_shift_illustration}
\end{figure}

\subsection{Further Analysis}\label{sec:experiments:further_analysis}
\paragraph{Cross-Domain Context Adaptation}
We aim to further evaluate Bifrost on cross-domain adaptation by adapting successful GSM8K trajectories toward LiveCodeBench. We expect that latent concepts such as logical reasoning and symbolic manipulation are shared between such domains. Tab.~\ref{tab:gsm8k_lcb} confirms that Bifrost remains effective on cross-domain adaptation, which outperforms the baseline methods. This result suggests that trajectories preserve context-independent patterns that can be adapted across domains.

\begin{table}[t]
    \caption{Bifrost effectiveness on cross-domain context adaptation: pass rates (\%) on LiveCodeBench code generation adapted from GSM8K math problem-solving trajectories.}
    \setlength{\tabcolsep}{13pt}
    \label{tab:gsm8k_lcb}
    \begin{center}
        \begin{small}
                \begin{tabular}{lccc}
                \toprule
                & \multicolumn{3}{c}{Llama-3.2-3B-Instruct} \\
                \cmidrule(lr){2-4}
                Method & pass@1 & pass@3 & pass@5 \\
                \midrule
                ICL        & 8.75 & 11.93 & 13.64 \\
                BoT        & 4.67 & 7.41 & 8.52 \\
                RISE       & 4.81 & 8.85 & 10.74 \\
                Paprika    & 9.33 & 12.37 & 13.70 \\
                \midrule
                \textbf{Bifrost} & \textbf{12.73} & \textbf{16.99} & \textbf{18.75}\\
                \bottomrule
                \end{tabular}
        \end{small}
    \end{center}
\end{table}

\paragraph{Number of ICL Examples Variants}
Fig.~\ref{fig:nicl_variants} demonstrates Bifrost behavior when leveraging different number of in-context examples. When evaluated on GSM8K, leveraging more in-context examples increases the agent effectiveness in general while it causes minor change in GPQA-Diamond. This can be explained as prior trajectories from ARC-Easy belongs to multi-domain science questions that are more challenging to find common transferable patterns to adapt toward GPQA-Diamond. Therefore, the skeptical effectiveness can stem from similarity of trajectory representations to extract and adapt useful patterns.

\begin{figure}[t]
    \centering
    \begin{subfigure}{0.3\linewidth}
        \centering
        \includegraphics[width=\linewidth]{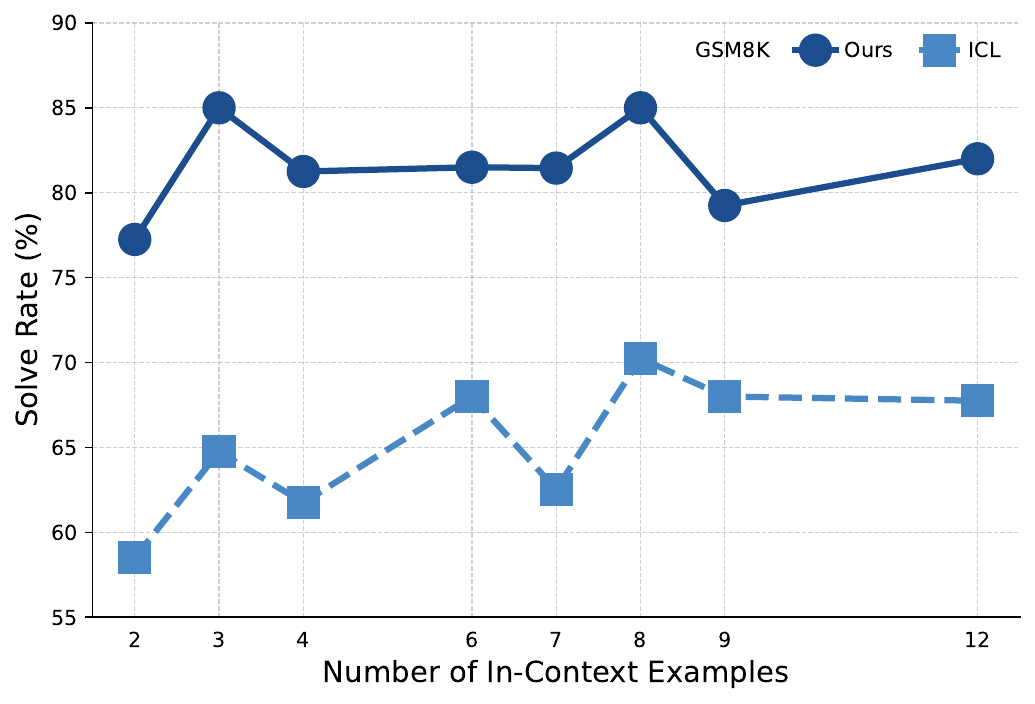}
    \end{subfigure}\hfill
    \begin{subfigure}{0.3\linewidth}
        \centering
        \includegraphics[width=\linewidth]{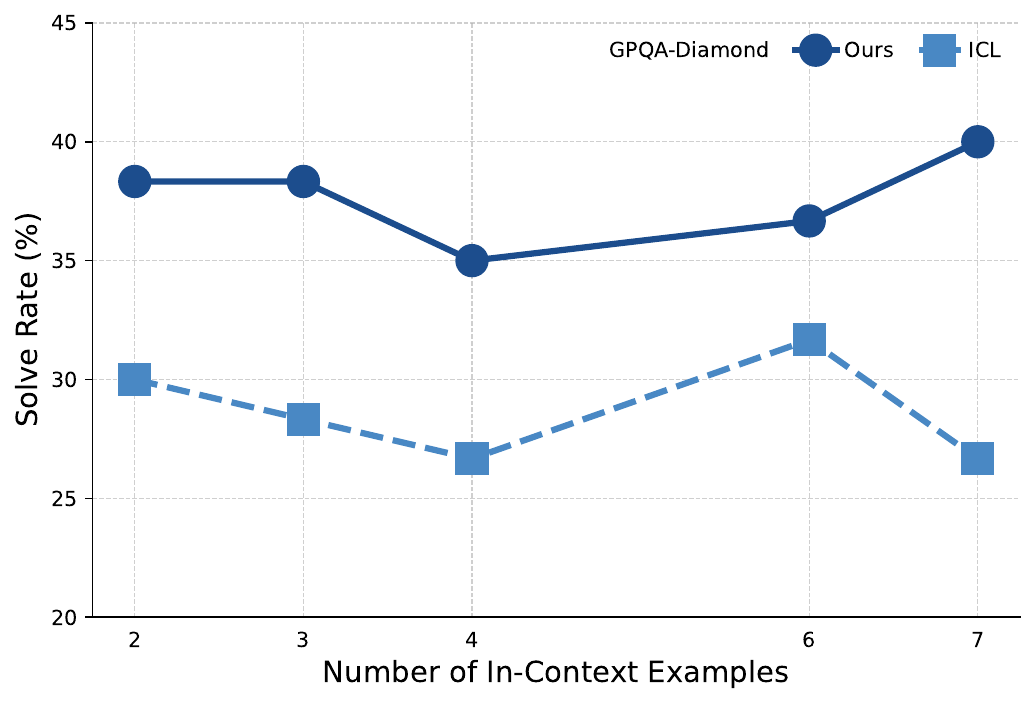}
    \end{subfigure}\hfill
    \begin{subfigure}{0.4\linewidth}
        \centering
        \includegraphics[width=\linewidth]{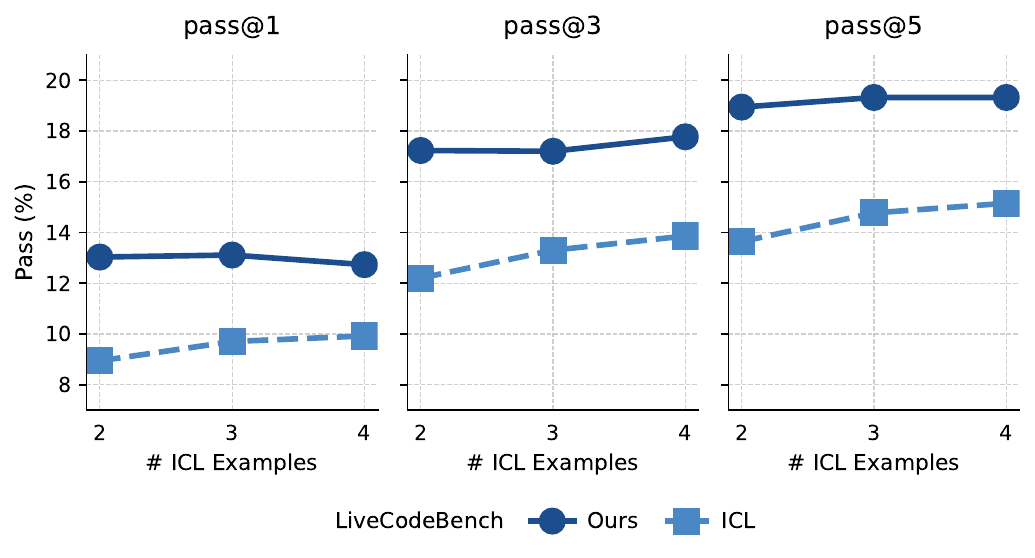}
    \end{subfigure}
    \caption{Bifrost effectiveness when leveraging different numbers of in-context examples.}
    \label{fig:nicl_variants}
\end{figure}

\paragraph{Steering Layer Position Effect}
To understand behavioral change under different layer intervention positions, we conduct trajectory steering across layers. As shown in Fig.~\ref{fig:steering_layer_variants}, agent effectiveness changes differently across task domains. For math problem solving, adapting mid-to-late layers brings optimal results, while the opposite behavior occurs in multi-domain science question answering. This behavior is explained by prior work \cite{venhoff2025understanding, zhao2025steering, turner2023steering} that steering early layers leads to more broad and foundational changes, while steering mid-to-late layers causes semantic concept changes. This well coincides with our experiments as multi-domain science requires foundational change, \textit{i.e., different science topic}, at early layers, while problem solving requires concept adaptation, \textit{i.e., leveraging arithmetic concepts}, that works best at mid-to-late layers. In code generation, steering mid-to-late layers also sightly increase the effectiveness.

\begin{figure}[h]
    \centering
    \begin{subfigure}{0.24\linewidth}
        \centering
        \includegraphics[width=\linewidth]{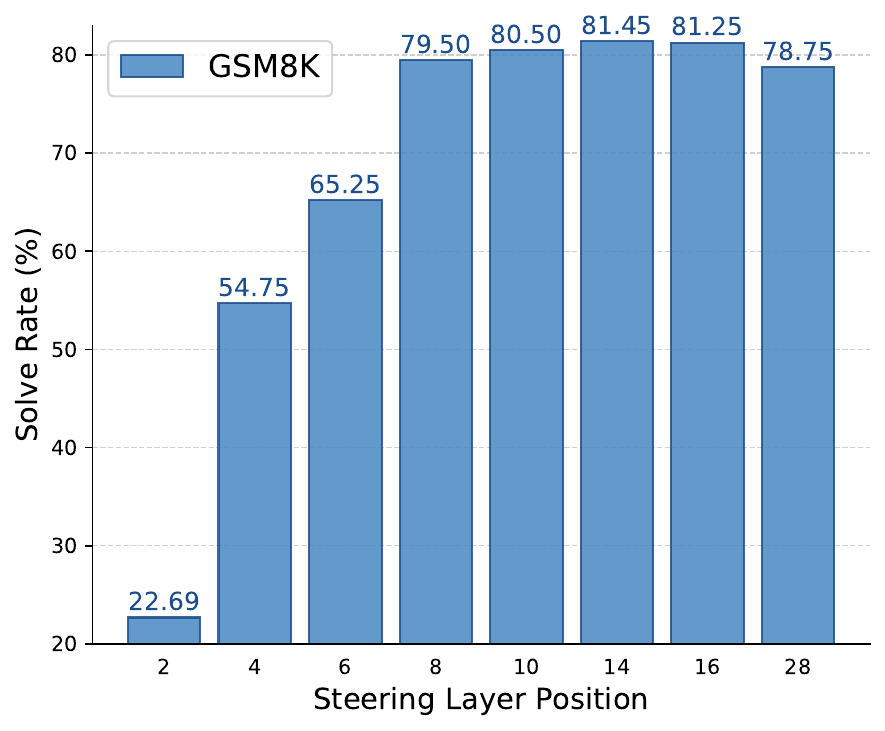}
    \end{subfigure}\hfill
    \begin{subfigure}{0.24\linewidth}
        \centering
        \includegraphics[width=\linewidth]{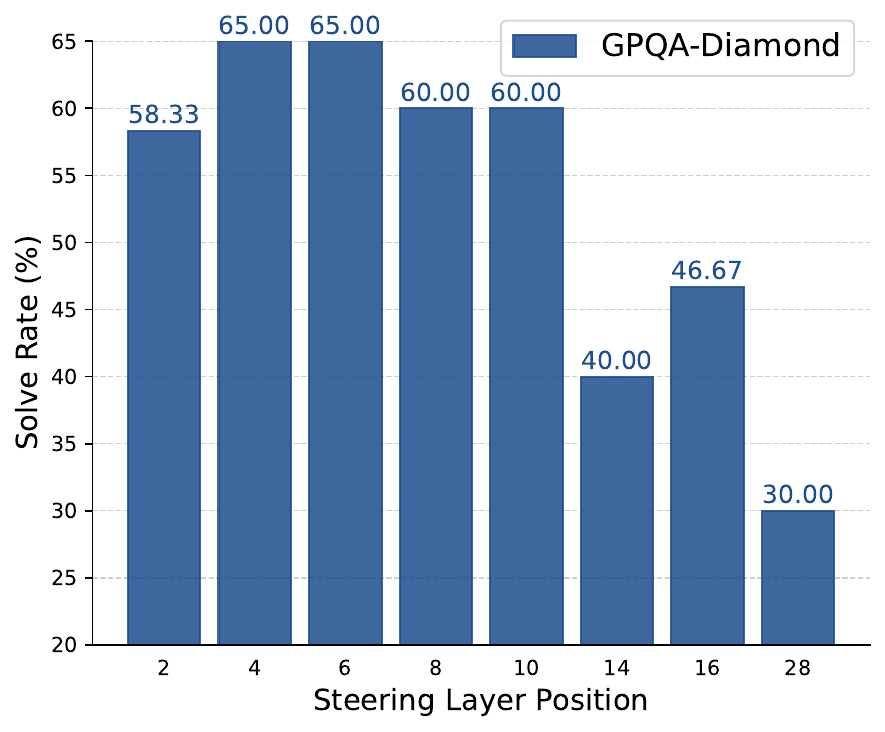}
    \end{subfigure}\hfill
    \begin{subfigure}{0.5\linewidth}
        \centering
        \includegraphics[width=\linewidth]{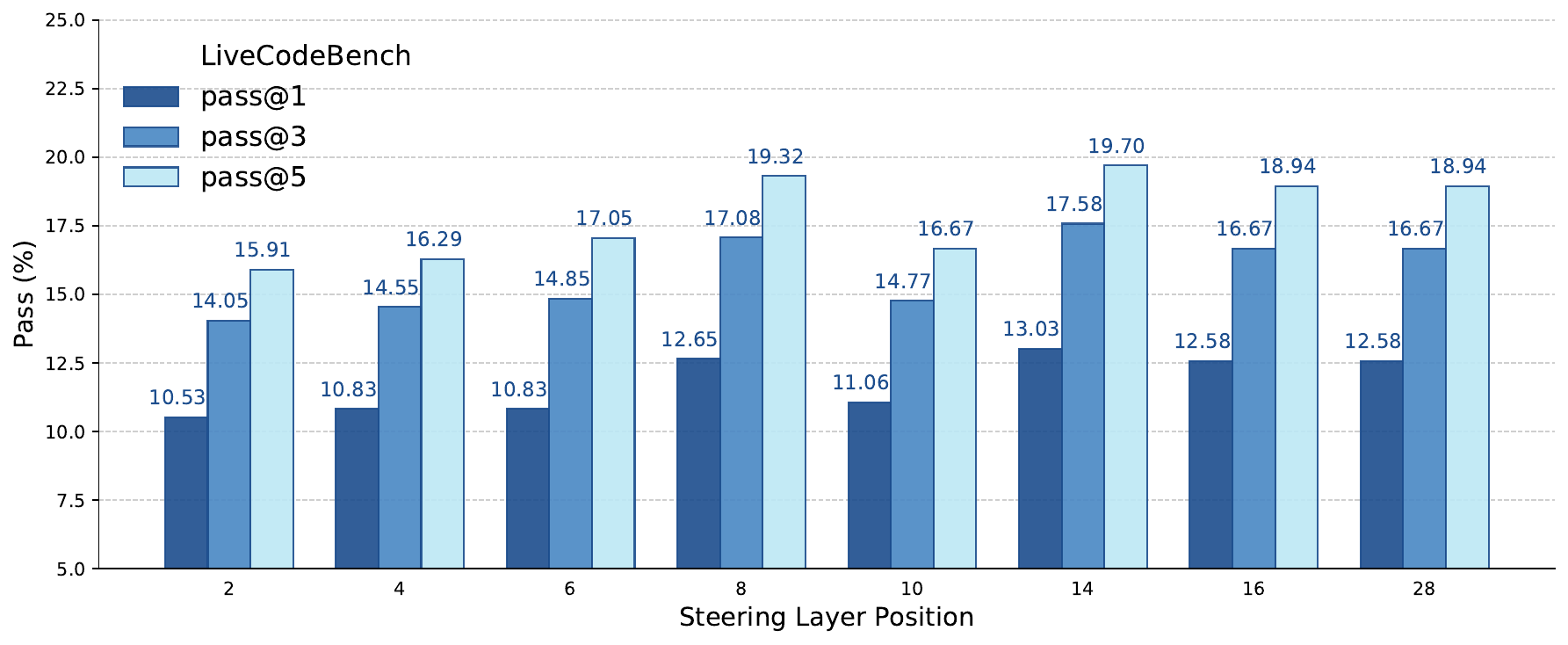}
    \end{subfigure}\hfill
    \caption{Bifrost with different steering layer positions.}
    \label{fig:steering_layer_variants}
\end{figure}

\paragraph{Suboptimal Trajectory Adaptation}
We reveal Bifrost ability to adapt suboptimal trajectories toward target tasks as shown in Fig.~\ref{fig:suboptimal_traj_adaptation}. Since contextual adaptation by Bifrost modifies trajectory behaviors by extracting transferable latent patterns, it is possible to leverage suboptimal trajectories, \textit{e.g., the ones producing incorrect answers in previous tasks}, as long as such trajectories share common patterns with the target tasks.
As shown in the figure, leveraging suboptimal trajectories surprisingly increases GSM8K solve rate; it could be explained as math-related patterns, such as arithmetic remains transferable even though incorrect in the previous tasks. In addition, GPQA-Diamond matches their original solve rate, while LiveCodeBench's effectiveness slightly degrades, possibly due to syntax errors in their defective trajectories, which could badly affect the adaptation.

\begin{figure}[t]
    \centering
    \begin{subfigure}{0.23\linewidth}
        \centering
        \includegraphics[width=\linewidth]{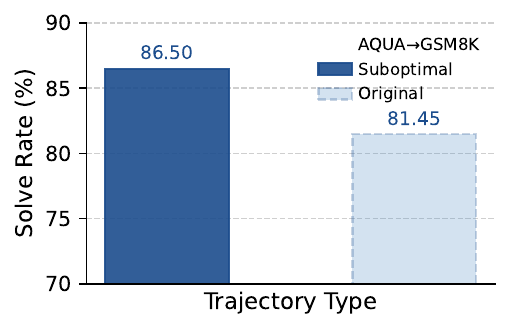}
    \end{subfigure}\hfill
    \begin{subfigure}{0.23\linewidth}
        \centering
        \includegraphics[width=\linewidth]{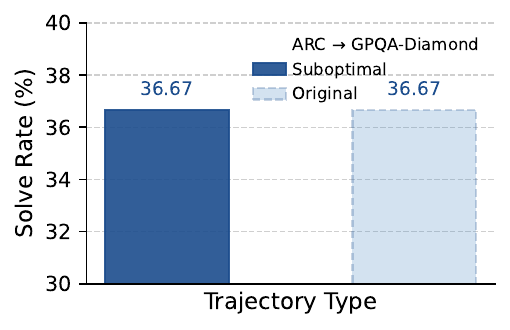}
    \end{subfigure}\hfill
    \begin{subfigure}{0.45\linewidth}
        \centering
        \includegraphics[width=\linewidth]{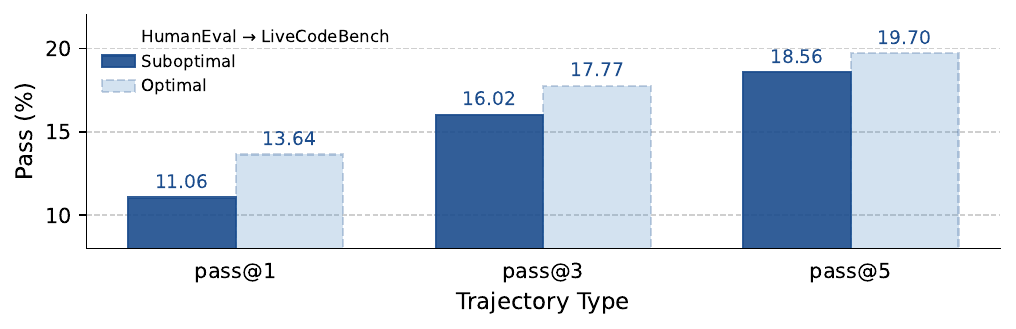}
    \end{subfigure}\hfill
    \caption{Bifrost when leveraging suboptimal prior trajectories.}
    \label{fig:suboptimal_traj_adaptation}
\end{figure}

\paragraph{Alternative Contextual Direction Methods}
We evaluate Bifrost via various techniques to extract contextual direction. In particular, we replace the existing per-task subtraction in Bifrost with \textit{Shared Subtraction} that computes a single contextual direction given prior trajectories and all target tasks. In addition, we leverage \textit{PCA} and \textit{Sparse Autoencoder} that are widely adopted in representation engineering~\cite{zou2023representation, hojer2025improving}. Tab.~\ref{tab:alternative_contextual_direction} shows the adaptation results in GSM8K. Shared subtraction degrades effectiveness by 4.2\% as a single steering vector hardly captures task specifics, while PCA further enhances Bifrost by 1.3\% thanks to better transferable pattern extraction under lower dimensional levels, and Sparse Autoencoder nearly matches the original Bifrost effectiveness.

\begin{table}[h]
    \caption{Integrability of Bifrost with alternative contextual direction extraction methods.}
    \label{tab:alternative_contextual_direction}
    \centering
    \begin{small}
        \begin{tabular}{lcc}
            \toprule
            \textbf{Methods} & \textbf{Solve Rate} & \textbf{Diff.}\\
            \midrule
            \textit{Original Bifrost}
                & \textit{81.45} 
                & \textit{--} \\
            \midrule
            Shared Subtraction
                &  77.25
                & \textcolor{elegantred}{-4.20} \\
            Principal Component Analysis (PCA)
                & 82.75
                & \textcolor{elegantgreen}{+1.30} \\
            Sparse Autoencoder                 
                & 80.95
                & \textcolor{elegantred}{-0.50}\\
            \bottomrule
        \end{tabular}
    \end{small}
\end{table}

\begin{table}[H]
    \caption{Ablation study on Bifrost.}
    \label{tab:ablation_study}
    \setlength{\tabcolsep}{11pt}
    \centering
    \begin{small}
        \begin{tabular}{lcc}
            \toprule
            \textbf{Ablation Configuration} & \textbf{Solve Rate} & \textbf{Diff.}\\
            \midrule
            \textit{Original Bifrost}
                & \textit{85.00} 
                & \textit{--} \\
            \midrule
            Without ICL prompt examples    
                & 84.25 
                & \textcolor{elegantred}{-0.75} \\
            Opposite contextual direction               
                & 77.50 
                & \textcolor{elegantred}{-7.50} \\
            Random contextual direction                 
                & 78.00 
                & \textcolor{elegantred}{-7.00} \\
            \bottomrule
        \end{tabular}
    \end{small}
\end{table}

\subsection{Ablation Study}\label{sec:experiments:ablation_study}
We conduct an ablation study on three factors that affect contextual adaptation direction. Specifically, we keep the existing contextual representation steering as is, but remove corresponding in-context examples from prompts. 
In addition, we perform untargeted contextual adaptations that steer the agent in directions either opposite to or unrelated to the target tasks. 
Tab.~\ref{tab:ablation_study} shows the agent effectiveness when adapted to GSM8K. Omitting in-context examples causes a slight decrease since most of the insightful patterns are embedded in their corresponding hidden states. Untargeted adaptation, such as opposite or random steering, significantly degrades agent effectiveness, which further confirms the role of target tasks as a contextual guidance to adapt meaningful trajectory behaviors.

\section{Conclusion}\label{sec:conclusion}
In this work, we proposed Bifrost, a training-free method that enables self-improving agents to improvise their prior trajectories across tasks under substantial context shifts. We revealed a context-trajectory correlation that allows Bifrost to precisely steer any previous trajectories toward target tasks by exploiting context differences. Both theoretical analyses and empirical results demonstrate that Bifrost leverages latent concepts acquired from prior in-context examples to adapt to new tasks, consistently outperforming existing self-improvement methods across diverse domains. Despite these promising results, our theoretical analysis is grounded in the Linear Representation Hypothesis, which may limit its applicability where more expressive nonlinear representations are required in practice. In the future, we plan to extend Bifrost to multi-context adaptation, where multiple target contexts are presented. Such a direction introduces an intensified complexity, which is essential for general-purpose agentic AI under multi-task applications.

\clearpage
\bibliographystyle{unsrt}  
\bibliography{references}

\newpage
\begin{appendices}

\appendix
\onecolumn
\begin{center}
 \rule{6.50in}{1.2pt}\\
 \vspace{0mm}
 {\Large\bf Appendix for ``Bifrost: Steering Strategic Trajectories\\ \vspace{0.06in}to Bridge Contextual Gaps for Self-Improving Agents''}
 \vspace{0mm}
 \rule{6.50in}{1.2pt}
\end{center}

\vskip 0.3in

In this Appendix, we first present related work in Sec.~\ref{sec:related_work}. We then further elaborate the hypothesis testing on the context-trajectory correlation in Sec.~\ref{appdix:context_traj_correlation}. Our proofs for the theoretical analysis is presented in Sec.~\ref{appdix:theoretical_analysis}. For the experiments, we discuss further details of experimental setup and demonstrate additional illustration and experimental results in Sec.~\ref{appdix:experiments}. Finally, Sec~\ref{appdix:prompts} presents the prompts used for Bifrost in our work.

\section{Related Work}\label{sec:related_work}
\paragraph{Self-Improving Agents}
Thanks to self-reflection ability \cite{weng2023large, madaan2023self, gou2023critic, pang2023language, bo2024reflective, zhang2024agent}, self-improving agents have emerged in leveraging their own experiences to solve complicated decision-making and reasoning problems through iterative prompt-based trial-and-error \cite{wang2023voyager, shinn2023reflexion, zhao2024expel, lingam2025enhancing, lin2025se, tajwar2025training, yu2024exact}, supervised finetuning and reinforcement learning via data self-generation \cite{zelikman2022star, zhang2023bootstrap, zhou2025self, pang2024kalm, bensal2025reflect, jin2025reveal, sarch2024vlm, sarukkai2025self}.
However, such approaches self-improve within target tasks \cite{shinn2023reflexion, yang2024buffer, lin2025se, zhao2024expel, yuan2025reinforce}, leaving cross-domain experience reuse strategies underexplored. While early work usually discards such insightful experiences to rely on supervised finetuning \cite{rebuffi2017icarl, zweiger2025self, narayanan2024aviary, pang2024kalm, costello2025think, amin2025escaping}, recent work attempts to address this by naively reusing prior trajectories \cite{shinn2023reflexion,lingam2025enhancing,zhao2024expel,fu2025agentrefine,madaan2023self}, generalizing thought templates \cite{yang2024buffer} as guidance, or relying on costly reinforcement learning to let agent unfold strategic and adaptable information 
\cite{bensal2025reflect, qu2024recursive, tajwar2025training, fu2025agentrefine, lin2025se, ge2025samule}. Unfortunately, they remain suboptimal when faced with contextual non-stationarity across tasks \cite{an2023context, lin2024dual, gai2025differential, afonin2025emergent}. Bifrost addresses this problem by contextual adaptation of prior trajectories without any training required.

\paragraph{Representation Engineering}
Previous work has demonstrated that representation manipulation under agent hidden states can change its behavior accordingly without any finetuning effort \cite{park2023linear, turner2023steering, zou2023representation, hu2025distribution, li2024open}. While most existing literature focus explicit control over specific emotion behaviors \cite{stolfo2024improving, turner2023steering}, it is more challenging to leverage representation engineering to enhance agent reasoning and problem-solving abilities, as such features are usually implicit \cite{li2023contrastive,  liu2023context, wang2024semantics, braun2025understanding, su2025activation}. Although recent work has revealed this possibility \cite{hojer2025improving}, a deterministic mechanism remains unexplored. In this work, we successfully exploit the representation engineering as a proxy of contextual adaptation to enable agents leverage their past experiences for target tasks effectively.

\paragraph{In-Context Learning}
Many empirical and theoretical results have shown that in-context examples can enhance LLM Agent effectiveness \cite{brown2020language, dherin2025learning, von2023transformers, ahn2023transformers, liu2024makes, tang2025theoretical, minetal2022rethinking, kossen2023context, park2024iclr, wies2023learnability, zhang2025can}. In addition, recent work theoretically proves ICL from the perspective of Bayesian Inference \cite{xie2021explanation, wakayama2025context, huang2024machine, wang2024noisegpt} where the agent approximates a latent concept from such examples to solve the target when such tasks are context-related \cite{panwar2023context, raventos2023pretraining}. However, ICL becomes ineffective when examples are specific to their previous tasks \cite{an2023context}. This phenomenon is even worse for self-improving agents reusing their past experiences as in-context examples due to significant context mismatch between previous and target tasks \cite{lin2024dual}, thus hindering generalization and further degrading agent behaviors \cite{fu2025agentrefine, wang2024can, goddard2025can, byunccl}. Bifrost preserves the simplicity of ICL while further strengthening it by adapting the latent concept of such examples toward the target task.

\section{Context-Trajectory Correlation}\label{appdix:context_traj_correlation}
\subsection{Hypothesis Testing and Experimental Demonstration Details}\label{appdix:hypothesis_testing_experiment}

\paragraph{Hypothesis Testing}
Our research question is whether previous-to-target trajectory steering preserves a linear correlation with trajectory change, such that a steered trajectory is more similar to the natural trajectory associated with the target query than the trajectory generated by a non-steered model for the original query. 
Let $(\hat{q}, \hat{a})$ be a target context pair consisting of a query and its corresponding natural trajectory, and let $(q, a)$ be a previous context counterpart. The contextual shift $\Delta$, computed using Eq.~\eqref{eq:contextual_direction}, is applied to the model to produce a steered trajectory $a_\mathrm{s}$.
We evaluate the relationship between context similarity and trajectory change by computing the Pearson correlation 
\begin{equation}\label{eq:appdix_pearson_corr}
    r = corr(s_q, s_a),    
\end{equation}
where $s_q$ denotes the context similarity obtained as
$$
s_q = \langle f(q), f(\hat{q}) \rangle,
$$
and $s_a$ represents the similarity of trajectory change,
$$
s_a = \langle f(a_\mathrm{s}), f(\hat{a}) \rangle - \langle f(a), f(\hat{a}) \rangle.
$$
Here, $f(\cdot)$ denotes an embedding extractor.

As a result, we formulate the following hypotheses,
\begin{itemize}
    \item $H_0$: There is no linear correlation between context shift and trajectory behavior shift.
    \item $H_1$: There is a confirmed linear correlation between context shift and trajectory behavior shift.
\end{itemize}

\paragraph{Experiments} To test the above hypothesis, we randomly select 100 samples from AQUA and GSM8K as two separate contexts. We then construct 50 sample pairs and apply trajectory steering from AQUA to GSM8K on each pair. For embedding extraction, we employ either their existing hidden states or the sentence transformer (all-MiniLM-L6-v2 in our work\footnote{\url{https://huggingface.co/sentence-transformers/all-MiniLM-L6-v2}}).
We first generate original trajectories for each context (using Llama-3.1-8B-Instruct), then apply steering vectors using Eq.~\eqref{eq:contextual_direction}. Finally, we measure p-values from Pearson correlation as shown in Eq.~\eqref{eq:appdix_pearson_corr}. 

Tab.~\ref{tab:appdix_hypothesis_testing_results_hidden_states} and~\ref{tab:appdix_hypothesis_testing_results_trajectory_embeddings} show the hypothesis testing results when using hidden state and trajectory embeddings, respectively. The result demonstrates a strong statistical significance across different settings of $\alpha$, showing our hypothesis of linear correlation between context shift and trajectory behavior shift exists.

\begin{table}[h]
    \centering
    \caption{Hypothesis testing results using hidden state embeddings}
    \label{tab:appdix_hypothesis_testing_results_hidden_states}
    \begin{small}
        \begin{tabular}{ccccccccc}
            \toprule
            $\alpha$          & 0.1     & 0.2     & 0.3     & 0.4     & 0.5     & 0.6     & 5.0     & 5.5     \\
            \midrule
            \textbf{p-value}  & 0.00418 & 0.00665 & 0.01085 & 0.00802 & 0.02019 & 0.02202 & 0.04279 & 0.01746 \\
            \bottomrule
        \end{tabular}
    \end{small}
\end{table}

\begin{table}[h]
    \centering
    \caption{Hypothesis testing results using trajectory embeddings}
    \label{tab:appdix_hypothesis_testing_results_trajectory_embeddings}
    \begin{small}
        \begin{tabular}{cccccc}
            \toprule
            $\alpha$          & 4.7     & 4.9     & 5.0     & 5.1     & 5.3     \\
            \midrule
            \textbf{p-value}  & 0.035   & 0.012   & 0.017   & 0.012   & 0.007   \\
            \bottomrule
        \end{tabular}
    \end{small}
\end{table}

\section{Theoretical Analysis}\label{appdix:theoretical_analysis}
\subsection{Proof of Lemma~\ref{lemma:laplace_update}}\label{appdix:proof_laplace_update}

After observing previous trajectories, we have approximated the latent concept posterior as a Gaussian
\begin{gather*}
    p(\phi \mid \mathcal{C}) \approx \mathcal{N}(\tilde{\phi}, \Sigma^{-1}), \\ 
    \text{where} 
    \quad \tilde{\phi} = \arg\max_\phi (\log p(\mathcal{C} \mid \phi) + \log p(\phi)),
    \quad
    \Sigma = \frac{1}{\sigma^2} \mathbf{F}^\top \mathbf{F},
\end{gather*}
where the mean $\tilde{\phi}$ is usually centered at the mode of previous trajectories. When observing a new target query $\hat{q} \in \hat{\mathcal{C}}$, we update the posterior and its mean accordingly,
\begin{align*}
    & p(\phi \mid \mathcal{C}, \hat{q}) \propto p(\hat{q} \mid \phi)\cdot p(\phi \mid \mathcal{C}). \\
    & \tilde{\phi}' = \arg\max_\phi\ (p(\hat{q} \mid \phi) + p(\phi \mid \mathcal{C})).
\end{align*}
Let $\mathcal{L}_{\phi} = p(\hat{q} \mid \phi) + p(\phi \mid \mathcal{C})$, we update the mean of new posterior by optimizing $\mathcal{L}_{\phi}$. To solve this, we will need to derive $p(\hat{q} \mid \phi)$ and $p(\phi \mid \mathcal{C})$.

Deriving $\log p(\hat{q} \mid \phi)$ by Taylor expansion to the new log-likelihood term,
\begin{gather*}
    \log p(\hat{q} \mid \phi) \approx \log p(\hat{q} \mid \tilde{\phi})+ g^\top (\phi - \tilde{\phi})\;+\;\frac{1}{2} (\phi - \tilde{\phi})^\top Q (\phi - \tilde{\phi}),\\
    \text{where}
    \quad g = \nabla_\phi \log p(\hat{q} \mid \phi) \Big|_{\phi=\tilde{\phi}}, 
    \quad Q = \nabla^2_\phi \log p(\hat{q} \mid \phi) \Big|_{\phi=\tilde{\phi}}.
\end{gather*}

Deriving $\log p(\phi \mid \mathcal{C})$ as
\begin{align*}
    \log p(\phi \mid \mathcal{C}) \;\approx\; \text{const} - \frac{1}{2} (\phi - \mu)^\top \Sigma^{-1} (\phi - \mu).
\end{align*}

As a result, $\mathcal{L}_\phi$ becomes
\begin{align*}
    \mathcal{L}_\phi \approx 
    - \frac{1}{2} (\phi - \tilde{\phi})^\top \Sigma^{-1} (\phi - \tilde{\phi}) 
    + g^\top (\phi - \tilde{\phi})\;+\;\frac{1}{2} (\phi - \tilde{\phi})^\top Q (\phi - \tilde{\phi}).
\end{align*}

To optimize $\mathcal{L}_\phi$, we take derivative w.r.t. $\phi$ and set to zero
\begin{align*}
    \nabla_\phi \mathcal{L}_\phi = -\Sigma(\phi-\tilde{\phi}) + g + Q(\phi-\tilde{\phi}) = 0.
\end{align*}

Rearranging 
and simplifying the single-step update by discarding $Q$, we have
\begin{equation}\label{eq:laplace_update_unsteer}
    \phi - \tilde{\phi} \approx \Sigma^{-1} g = \Sigma^{-1}\nabla_\phi \log p(\hat{q} \mid \phi) \Big|_{\phi=\tilde{\phi}}.
\end{equation}
Therefore, adding the target query $\hat{q} \in \hat{\mathcal{C}}$ will update the latent concept posterior as
\begin{equation}
    p(\phi \mid \mathcal{C}, \hat{q}) \approx \mathcal{N}\left(
        \tilde{\phi} + \Sigma^{-1} \nabla_\phi \log p(\hat{q} \mid \phi), \Sigma^{-1}
    \right).
\end{equation}

Under the Linear Representation Hypothesis, $\log p(\hat{q} \mid \phi)$ from Eq.~\eqref{eq:laplace_update_unsteer} is generally a linear function of the hidden state expressed as
\begin{align*}
    \log p(\hat{q} \mid \phi) \approx \mathbf{h}(\hat{q} ;\phi) \mathbf{W},
\end{align*}
where $\mathbf{h}$ is the hidden state of LLM $\mathcal{M}$ and $\mathbf{W}$ is a linear projection.

Treating $\mathbf{h}$ as a function of $\phi$, derive $\log p(\hat{q} \mid \phi)$,
\begin{align*}
    \nabla_\phi \log p(\hat{q} \mid \phi) = \frac{\partial \log p}{\partial \mathbf{h}} \cdot \frac{\partial \mathbf{h}}{\partial \phi}.
\end{align*}

The partial derivative of $\mathbf{h}$ can be approximated by a first-order Taylor expansion, assuming that $\mathbf{h}$ is locally smooth in $\phi$, and $\phi^*$ is the ideal concept of the target task context $\hat{\mathcal{C}}$, we have
\begin{align*}
    \mathbf{h}(\phi^*) \approx \mathbf{h}(\phi) + \frac{\partial \mathbf{h}}{\partial \phi}(\phi^* - \phi).
\end{align*}

Since $\Delta \approx \mathbf{h}(\phi^*) - \mathbf{h}(\phi)$, we have
\begin{align}\label{eq:derive_delta}
    \Delta &\approx \frac{\partial\mathbf{h}}{\partial \phi} (\phi^* - \phi). \\
    &\approx \mathbf{J} \cdot \Sigma^{-1} \nabla_\phi \log p(\hat{q} \mid \phi) \quad \quad \text{(from Eq.~\eqref{eq:laplace_update_unsteer})}\\
    &= \Sigma^{-1} \cdot \mathbf{J} \nabla_\phi \log p(\hat{q} \mid \phi).
\end{align}

Rearranging,
\begin{align}\label{eq:infer_p_delta}
    \Sigma^{-1} \nabla_\phi \log p(\hat{q} \mid \phi) = \mathbf{J}^\top \Delta = P(\Delta).
\end{align}

Therefore, applying a trajectory steering $\Delta$ shifts the mean toward an ideal concept of the target context, the steered concept posterior is expressed as
\begin{align*}
    p_\mathrm{s}(\phi \mid \mathcal{C}, \hat{q}) \approx \mathcal{N}\left(\tilde{\phi} + \alpha P(\Delta), \Sigma^{-1}\right).
\end{align*}


\subsection{Proof of Theorem~\ref{theorem:risk_bound}}\label{appdix:proof_risk_bound}

Let 
$y^* = f_\phi(\hat{q})$ be the true label of target query $\hat{q}$, 
$\hat{y}_\mathrm{s} = \mathcal{M}_\mathrm{s}(\mathcal{C} \circ\hat{q})$ be the prediction of the steered LLM agent, 
$p_\mathrm{s} = p_\mathrm{s}(\phi \mid \mathcal{C}, \hat{q})$ be the latent concept posterior of the steered model. 
Our target bound is
\begin{align*}
    \mathbb{E}\left[\mathcal{L}(\hat{y}_\mathrm{s}, y^*)\right] 
    &= \mathbb{E}_{\phi \sim p_\mathrm{s}}\left[\mathcal{L}(f_\phi(\hat{q}), y^*)\right] \\
    &\leq \mathcal{L}\left( \mathbb{E}_{\phi \sim p_\mathrm{s}}\left[f_\phi(\hat{q})\right], y^* \right) + Var[f_\phi(\hat{q})] \quad \text{(Jensen inequality with smooth and convex $\mathcal{L}$)}\\
    &\approx \mathcal{L}(f_{\mathbb{E}[\phi]}(\hat{q}), y^*) + \mathcal{O}(Var[f_\phi(\hat{q})]).
\end{align*}

The target bound consists of the bias and variance terms. 
We first bound the variance term. From Lemma~\ref{lemma:laplace_update} we have $p_\mathrm{s}(\phi \mid \cdot) \approx \mathcal{N}(\mu, \Sigma^{-1})$. Since $f_\phi$ is locally Lipschitz in $\phi$ with a constant $G$, we have
\begin{align*}
    Var[f_\phi(\hat{q})] \; &\leq \;  G^2 \; \mathrm{Tr}(\Sigma^{-1})
    \; = \; G^2 \; \mathcal{O}(\frac{1}{k})
    \; \leq \; \mathcal{O}(\frac{1}{k}).
\end{align*}

Next, we bound the bias term,
\begin{align*}
    \mathcal{L}(f_{\mathbb{E}[\phi]}(\hat{q}), y^*) &\leq C\;||\mu - \phi^*||^2 \quad \text{($\mathcal{L}$ is smooth)}\\
    &= ||\alpha P(\Delta) - (\phi^* - \tilde{\phi})||^2\\
    &\leq \mathcal{O}\left(||\alpha P(\Delta) - (\phi^* - \tilde{\phi})||^2\right)
\end{align*}

As a result, we obtain the bound of the steered model as
\begin{align*}\label{eq:bound_steered}
    \mathbb{E}\left[\mathcal{L}(\hat{y}_\mathrm{s}, y^*)\right] 
    &\leq \mathcal{O}\left(\frac{1}{k} + ||\alpha P(\Delta) - (\phi^* - \tilde{\phi})||^2\right).
\end{align*}

Since $\Delta \propto (\phi^* - \tilde{\phi})$, thus $\alpha P(\Delta) = c(\phi^* - \tilde{\phi})$ for a scalar $c > 0$, then
\begin{align*}
    \mathbb{E}\left[\mathcal{L}(\hat{y}_\mathrm{s}, y^*)\right]  \leq \mathcal{O}(\frac{1}{k}).
\end{align*}

For an ordinary model, its prediction $y=\mathcal{M}(\hat{q})$ induces a broader pretrained latent concept, then
\begin{align*}
    \mathbb{E}[\mathcal{L}(y, y^*)] \leq \mathcal{O}(1).
\end{align*}

Subtracting the bounds, we have
\begin{align*}
     \mathbb{E}[\mathcal{L}(y, y^*)] - \mathbb{E}\left[\mathcal{L}(\hat{y}_\mathrm{s}, y^*)\right]  \geq \mathcal{O}(1) - \mathcal{O}(\frac{1}{k}).
\end{align*}

Rearranging
\begin{align*}
    \mathbb{E}\left[\mathcal{L}(\hat{y}_\mathrm{s}, y^*)\right] \leq \mathbb{E}[\mathcal{L}(y, y^*)] - \Omega\left(\frac{1}{k}\right).
\end{align*}

\section{Experiments}\label{appdix:experiments}
\subsection{Experimental Setup}\label{appdix:experiments:setup}

\paragraph{Context Shift Experimental Setup} 
To evaluate trajectory leverage of self-improving agents under substantial context shifts, we setup source - target dataset pairs. In particular, a self-improving agent will solve tasks on a source dataset to collect successful trajectories (and/or unsuccessful ones depending on baselines), then leverage such trajectories to adapt toward the target task in the sense that there is no ground-truth available on such tasks of the target dataset. In this way, the self-improving agent has to employ their past experiences learned from previous tasks via such trajectories to solve the target tasks. We now present how this setup applies to Bifrost and the baselines in our work.

\paragraph{Bifrost Setup} 
Bifrost does not require any particular self-improving agent mechanism to solve tasks and collect corresponding trajectories. In our experiment, we use ordinary LLM agents as mentioned to solely solve source tasks and subsequently collect successful trajectories. 

We let the agent solve a subset of 20 tasks when using Llama-3.1-8B, and 50 tasks when using Llama-3.2-3B and Llama-3.2-1B, for all source tasks in our experiments (AQUA, ARC-Easy, HumanEval). We subsequently collect corresponding successful trajectories and randomly pick a number of them as in-context examples to solve target tasks. In particular, we randomly pick 7 examples in \textit{AQUA $\rightarrow$ GSM8K} and \textit{ARC-Easy $\rightarrow$ GPQA-Diamond} experiments, and 5 examples in \textit{HumanEval $\rightarrow$ GSM8K} experiments. For the steering coefficient $\alpha$, we get the optimal value by conducting a grid search $\alpha \in \{1,2,3,4\}$. In our main experiments, we conduct the steering on residual stream of a single middle layer. In particular, we steer 20\textsuperscript{th} layer in Llama-3.1-8B, 14\textsuperscript{th} layer in Llama-3.2-3B, and 10\textsuperscript{th} layer in Llama-3.2-1B. Steering on other layers are discussed in Sec.~\ref{sec:experiments:further_analysis} and Sec.~\ref{appdix:experiments:steering_layer_positions}.

\paragraph{Baselines Setup} 
We compare Bifrost with the following baselines.
\begin{itemize}
    \item \textit{Chain-of-Though (CoT)} \cite{wei2022chain} prompts an agent to perform step-by-step reasoning. We directly apply CoT to the target tasks.
    \item \textit{In-Context Learning (ICL)} \cite{brown2020language}, we let an agent solve tasks from a source dataset to collect successful trajectories. We then leverage such trajectories as in-context examples in the prompts to solve target tasks. We use 7 in-context examples for both the math problem-solving task and the multiple-choice question answering task, and we use 5 in-context examples for the code generation task.
    \item \textit{Reflexion} \cite{shinn2023reflexion} is able to iteratively refine its own trajectories given environment feedback. We directly apply Reflexion to the target tasks. We allow up to 5 attempts for math problem solving and code generation problems, as in their original work. For the multi-choice question answering problem, we allow 3 attempts since more attempts can enable the agent to easily get the correct answer with 4-choice questions, which is unfair to other methods.
    \item \textit{Buffer of Thoughts (BoT)} \cite{yang2024buffer} creates thought templates when solving tasks as guidance in future tasks. We let BoT solve the source dataset tasks, subsequently creating corresponding though templates, which then serve as guidance in target tasks. Following their original work \cite{yang2024buffer}, we let BoT continuously update its thought templates while solving the target tasks.
    \item \textit{RISE} \cite{qu2024recursive} teaches an agent to self-improve by learning each fail-to-correct transition of trajectories via reinforcement learning. In our experiments, we let RISE solve source tasks and subsequently perform reinforcement learning on corresponding trajectories. We then leverage the finetuned agent to solve target tasks. We evaluate \textit{RISE} within a single attempt and 5 attempts, respectively, as recommended by the original work \cite{qu2024recursive}. 
    \item \textit{Diversity of Thought (DoT)} \cite{lingam2025enhancing} lets an agent generate diverse trajectories and self-reflect on given environment feedback. Successful trajectories are stored in a memory bank as is and subsequently retrieved as in-context examples when solving similar tasks. We let DoT solve source tasks, and employ it to solve the target tasks.
    \item \textit{Paprika} \cite{tajwar2025training} learns strategic information by letting an agent self-explore and adjust its trajectories given environment feedback via a combined Supervised Finetuning and Direct
    Preference Optimization. Such information is prone to out-of-domain as shown in their original paper \cite{tajwar2025training}. Similar to RISE, we learn Paprika on source tasks and apply it to the target tasks.
\end{itemize}

\paragraph{Evaluation  Metrics}
For \textit{AQUA $\rightarrow$ GSM8K} and \textit{ARC-Easy $\rightarrow$ GPQA-Diamond} experiments, we allow a single generation for each question, except \textit{Reflexion} and \textit{RISE} baselines, and perform exact matching between agent answers and ground-truths, then measure the solve rate. For \textit{HumanEval $\rightarrow$ GSM8K}, we generate 5 solutions for each problem and measure functional correctness following previous work \cite{chen2021codex, jain2024livecodebench} to determine the solutions passing test cases via pass@1, pass@3, and pass@5.

\paragraph{Datasets and LLMs}
We provide the dataset and LLM details as follows.
For datasets, we leverage AQUA (Algebra Question Answering with Rationales)\footnote{\url{https://github.com/google-deepmind/AQuA}}, 
GSM8K\footnote{\url{https://huggingface.co/datasets/openai/gsm8k}}, 
ARC-Easy\footnote{\url{https://huggingface.co/datasets/allenai/ai2_arc}}, 
GPQA-Diamond\footnote{\url{https://huggingface.co/datasets/fingertap/GPQA-Diamond}}, 
HumanEval\footnote{\url{https://huggingface.co/datasets/openai/openai_humaneval}}, and 
LiveCodeBech\footnote{\url{https://huggingface.co/datasets/openai/openai_humaneval}} (release\_v5) with 880 problems.
For the LLMs, we leverage different model sizes, including Llama 
1B\footnote{\url{https://huggingface.co/meta-llama/Llama-3.2-1B-Instruct}}, 
3B\footnote{\url{https://huggingface.co/meta-llama/Llama-3.2-3B-Instruct}}, and 
8B\footnote{\url{https://huggingface.co/meta-llama/Llama-3.1-8B-Instruct}} Instruct versions.

\subsection{Latent Concept Shift Illustration}\label{appdix:experiments:latent_concept_shift}

We show more steering effect illustrations in AQUA $\rightarrow$ GSM8K (Fig.~\ref{fig:appdix_steering_effect_gsm8k}), ARC-Easy $\rightarrow$ GQPA-Diamond (Fig.~\ref{fig:appdix_steering_effect_gpqa}), and HumanEval $\rightarrow$ LiveCodeBench (Fig.~\ref{fig:appdix_steering_effect_lcb}).

\begin{figure}[h]
    \centering
    \begin{subfigure}{0.31\linewidth}
        \centering
        \includegraphics[width=\linewidth]{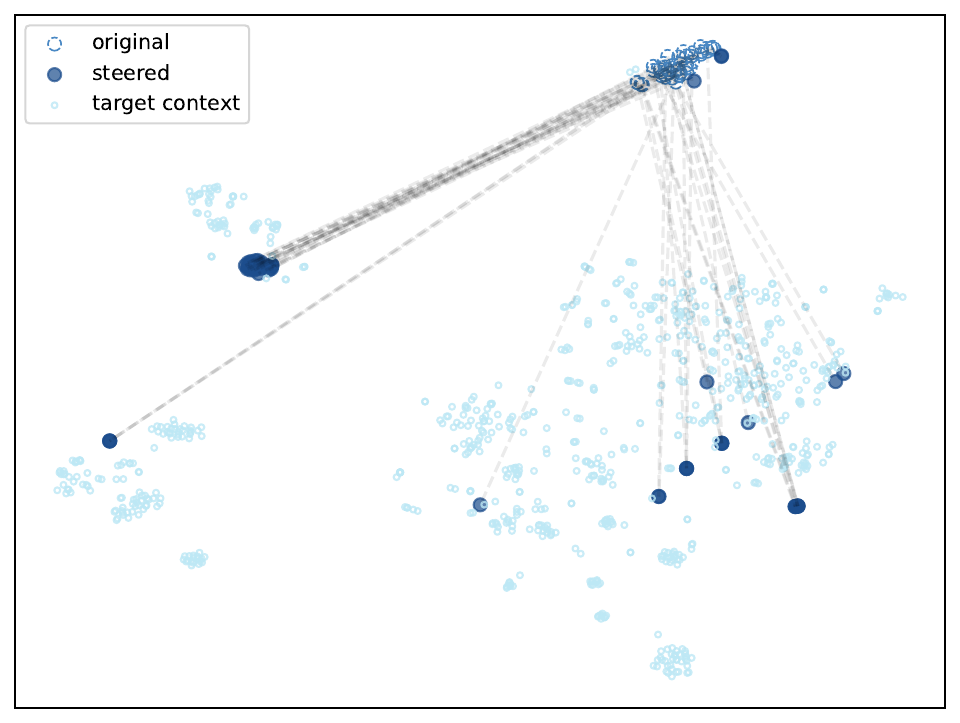}
    \end{subfigure}
    \begin{subfigure}{0.31\linewidth}
        \centering
        \includegraphics[width=\linewidth]{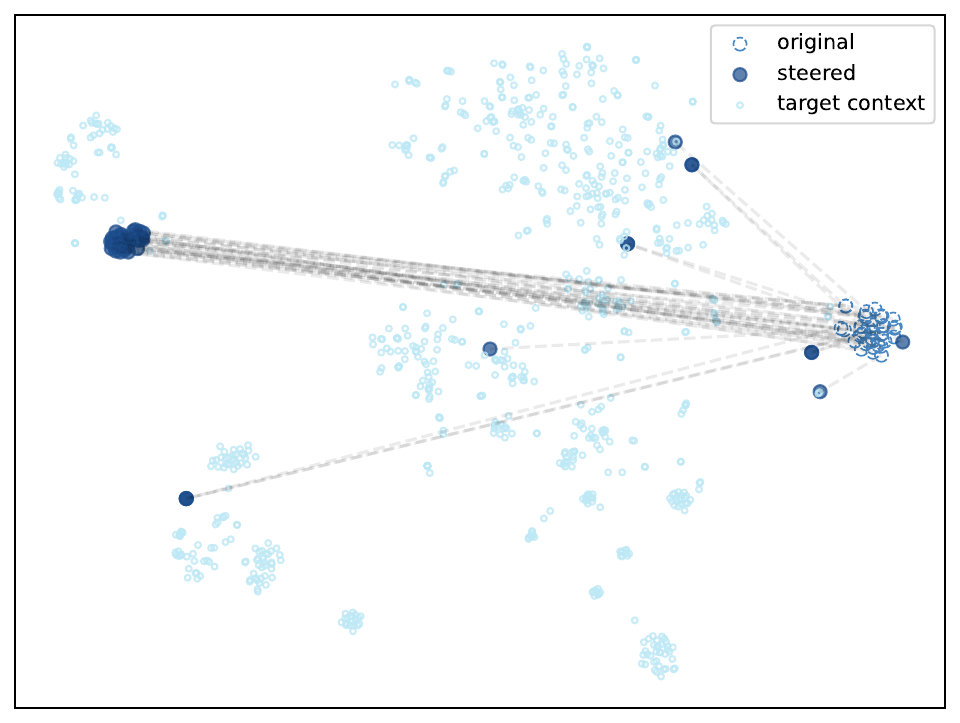}
    \end{subfigure}
    \begin{subfigure}{0.31\linewidth}
        \centering
        \includegraphics[width=\linewidth]{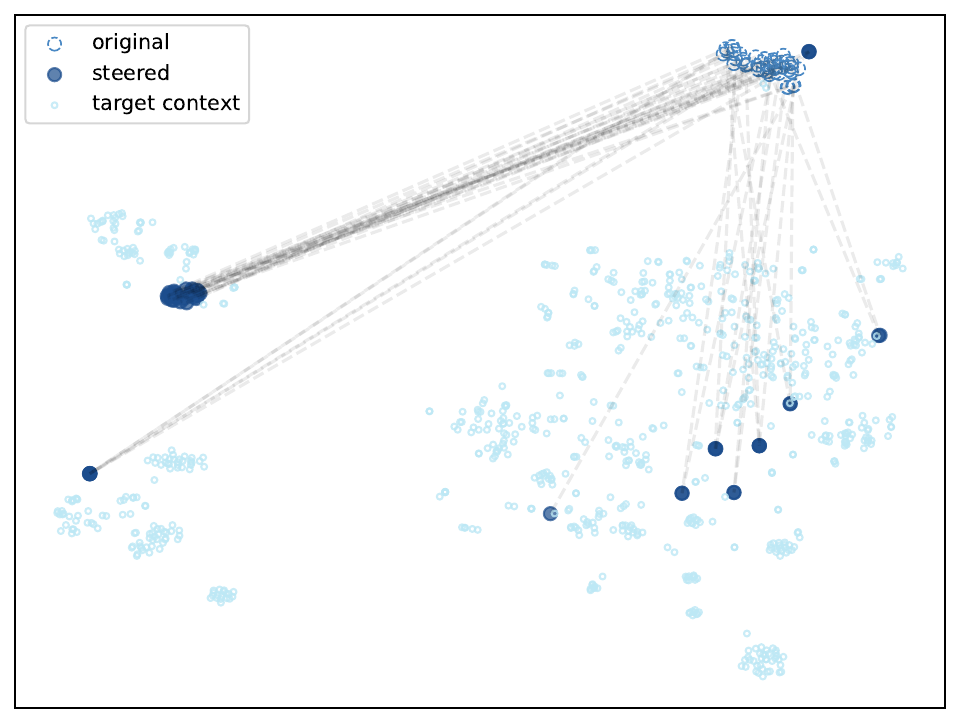}
    \end{subfigure}
    \begin{subfigure}{0.31\linewidth}
        \centering
        \includegraphics[width=\linewidth]{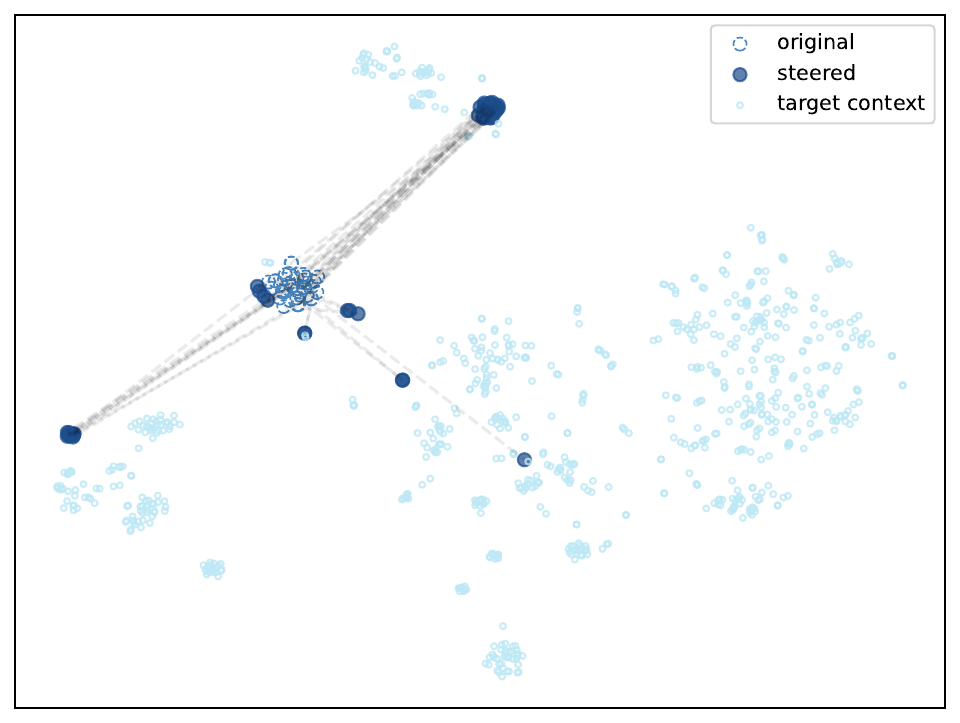}
    \end{subfigure}
    \begin{subfigure}{0.31\linewidth}
        \centering
        \includegraphics[width=\linewidth]{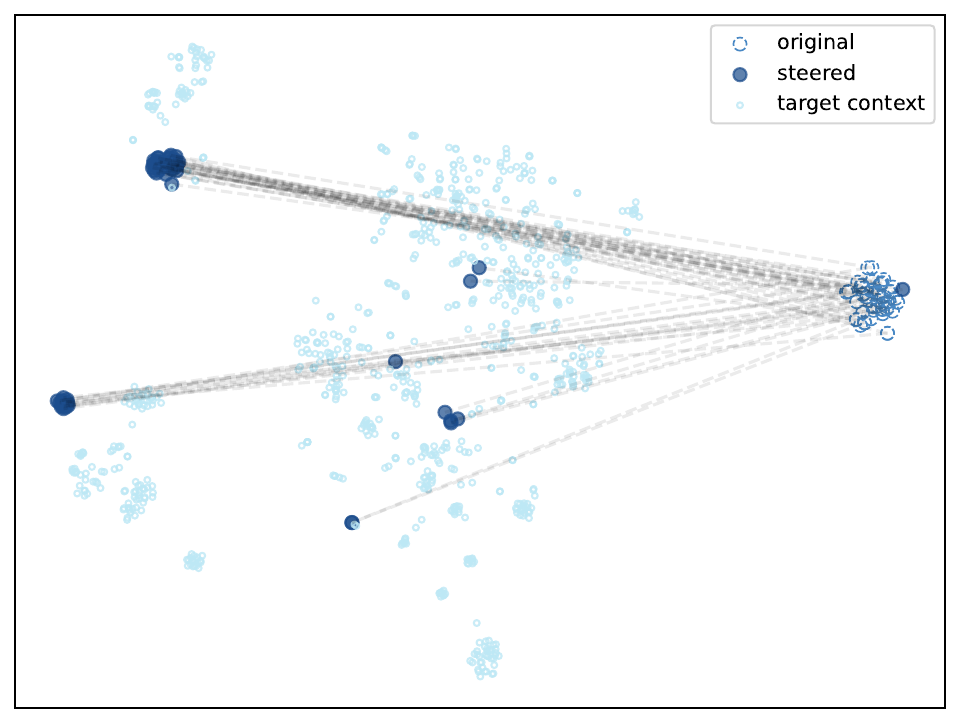}
    \end{subfigure}
    \begin{subfigure}{0.31\linewidth}
        \centering
        \includegraphics[width=\linewidth]{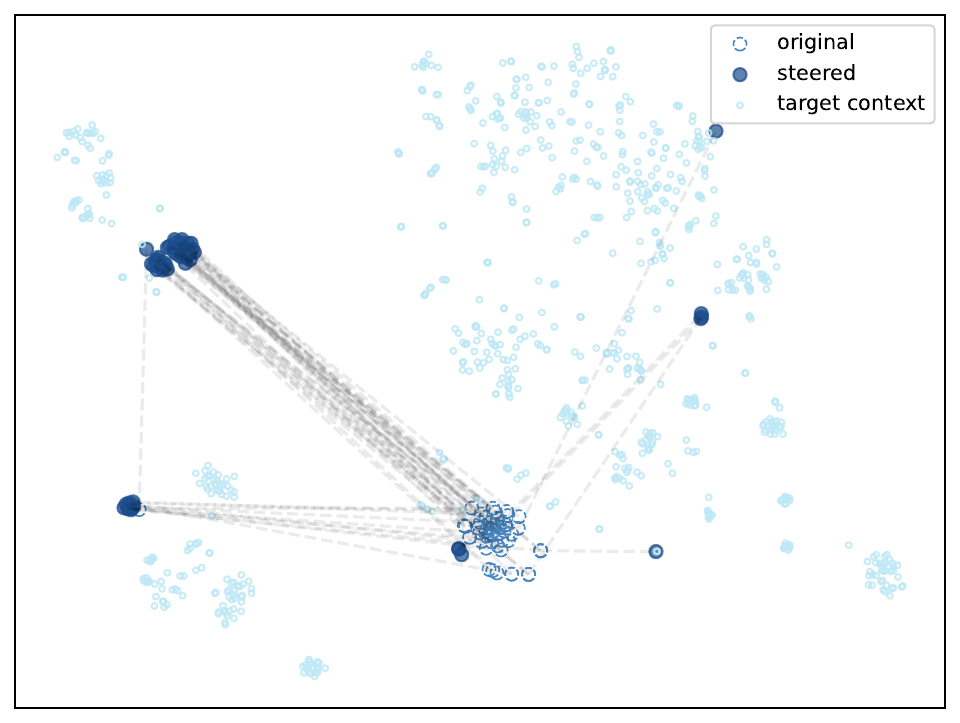}
    \end{subfigure}
    \begin{subfigure}{0.31\linewidth}
        \centering
        \includegraphics[width=\linewidth]{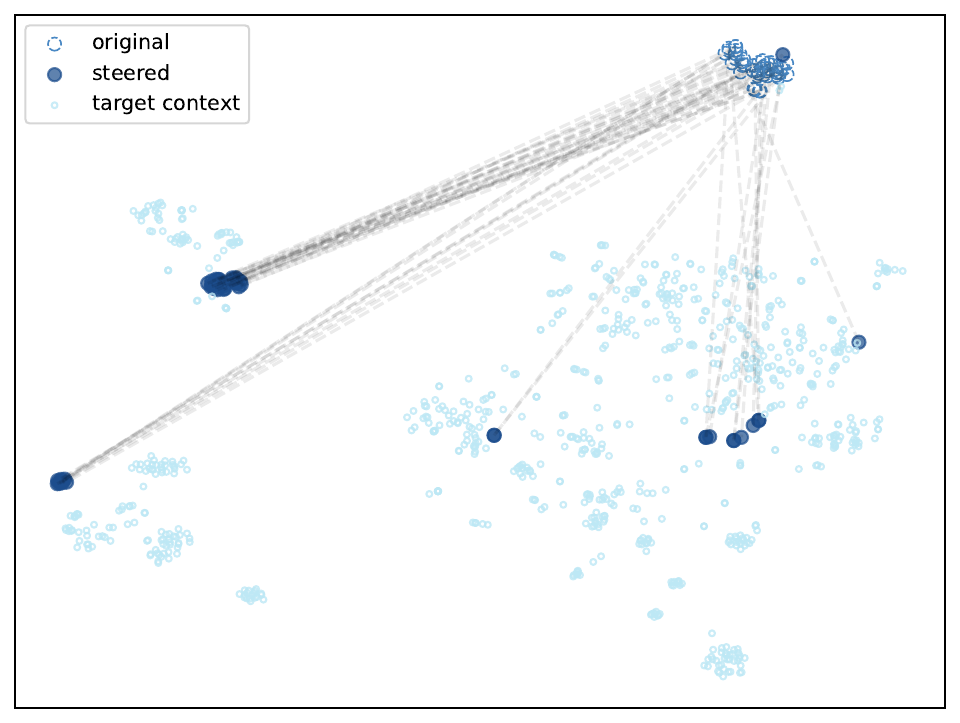}
        \caption{}
    \end{subfigure}
    \caption{Illustration of Latent Concept Shift from AQUA to GSM8K.}
    \label{fig:appdix_steering_effect_gsm8k}
\end{figure}

\begin{figure}[h]
    \centering
    \begin{subfigure}{0.31\linewidth}
        \centering
        \includegraphics[width=\linewidth]{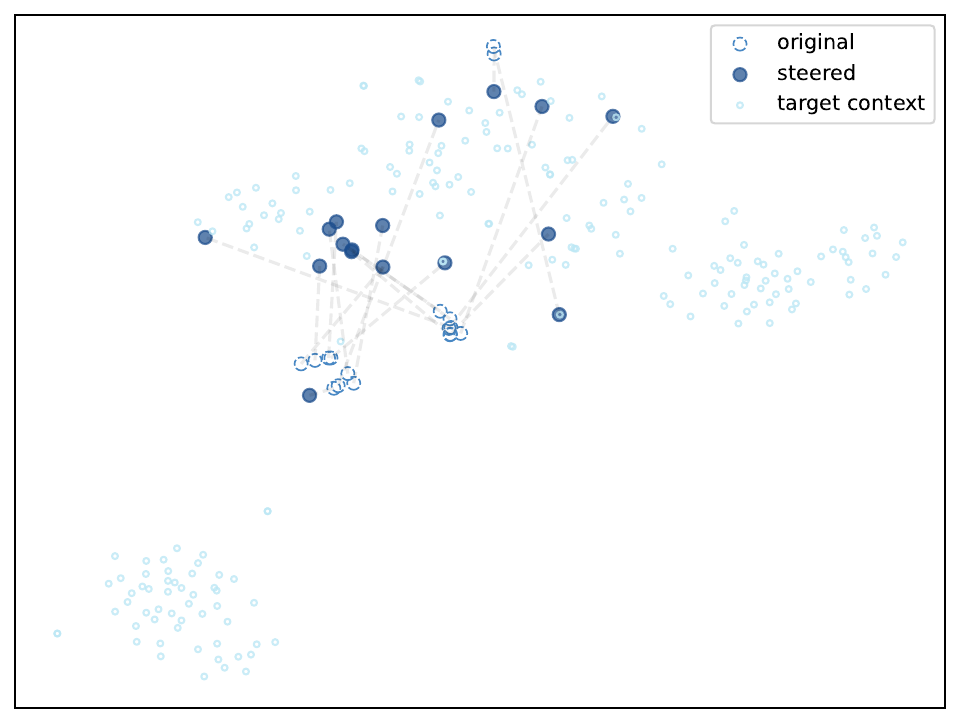}
    \end{subfigure}
    \begin{subfigure}{0.31\linewidth}
        \centering
        \includegraphics[width=\linewidth]{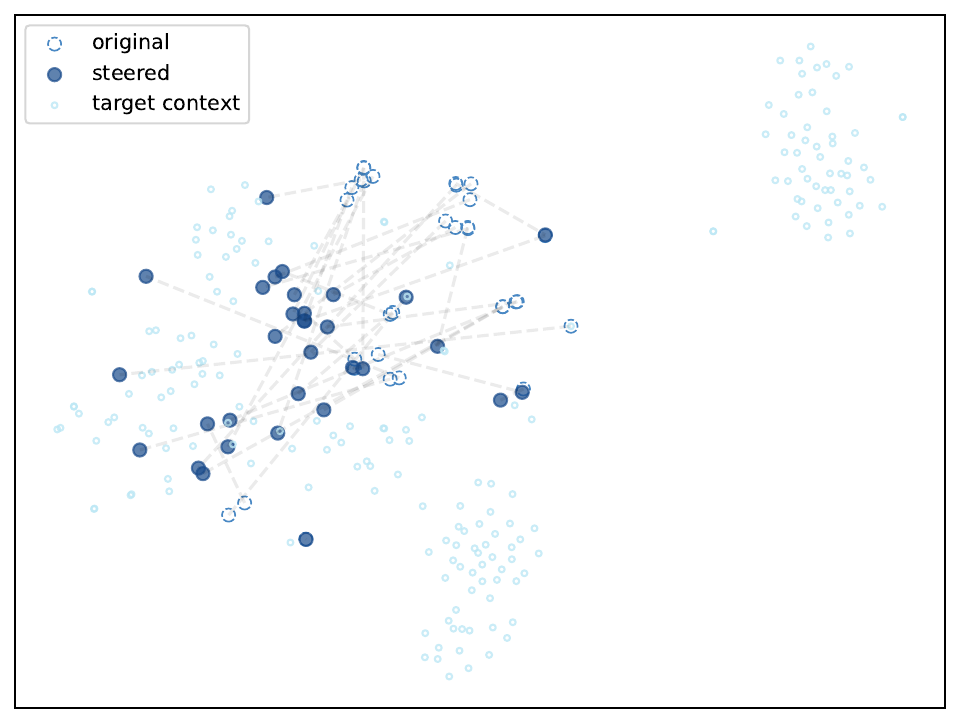}
    \end{subfigure}
    \begin{subfigure}{0.31\linewidth}
        \centering
        \includegraphics[width=\linewidth]{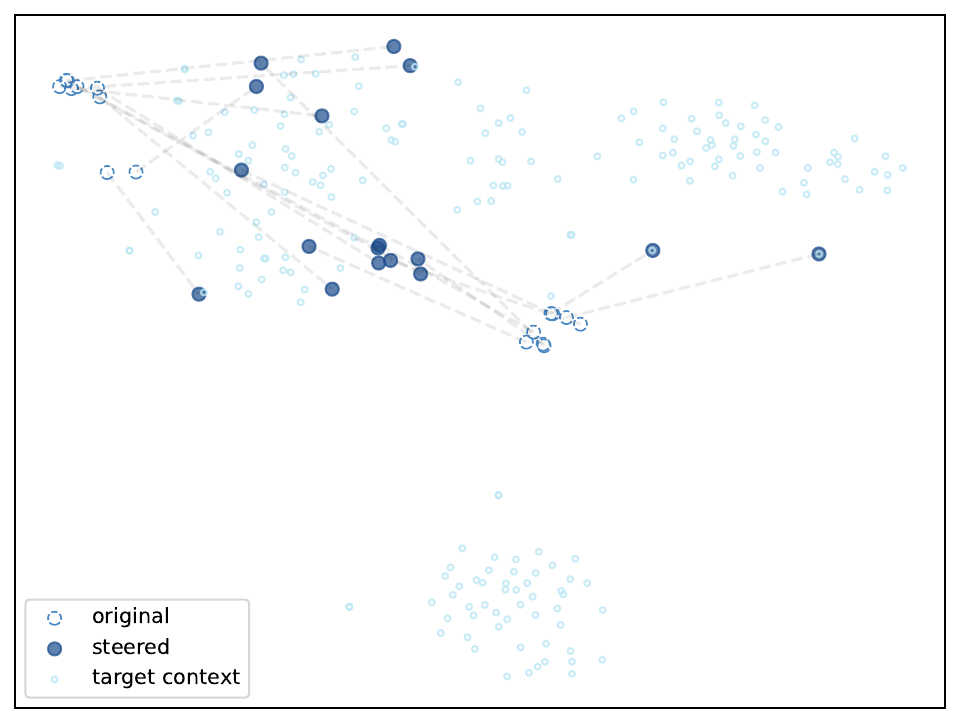}
    \end{subfigure}
    \begin{subfigure}{0.31\linewidth}
        \centering
        \includegraphics[width=\linewidth]{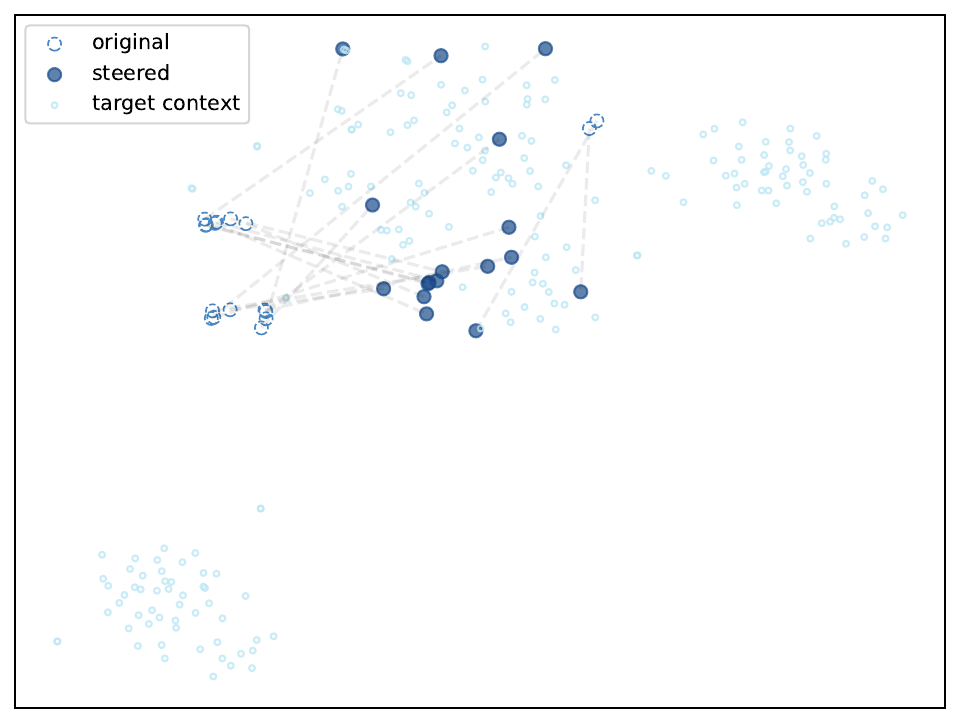}
    \end{subfigure}
    \begin{subfigure}{0.31\linewidth}
        \centering
        \includegraphics[width=\linewidth]{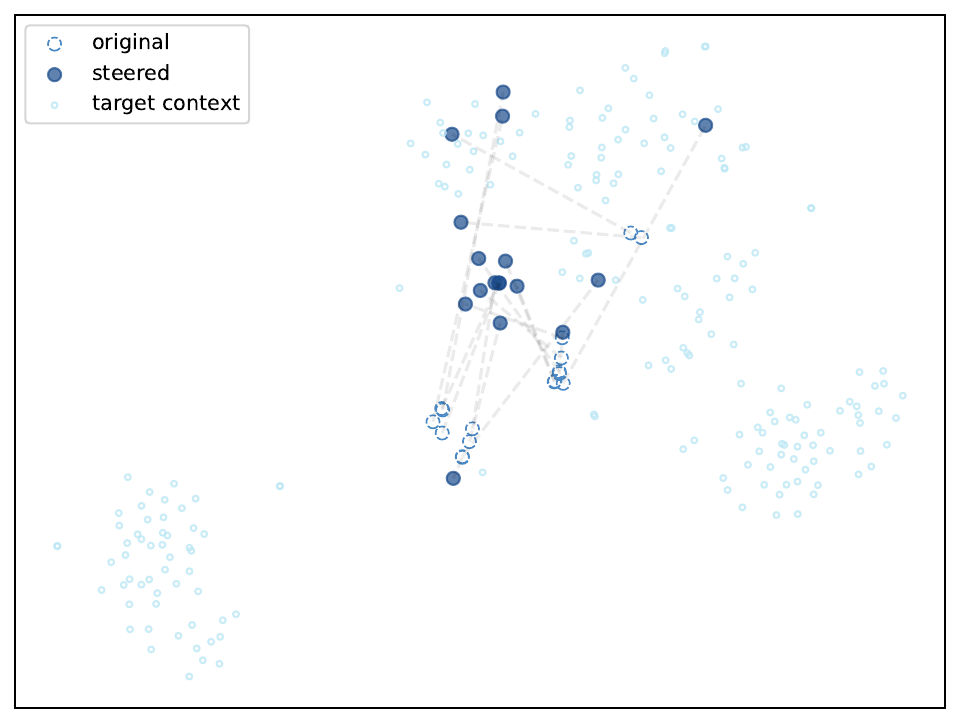}
    \end{subfigure}
    \caption{Illustration of Latent Concept Shift from ARC-Easy to GPQA-Diamond.}
    \label{fig:appdix_steering_effect_gpqa}
\end{figure}

\begin{figure}[h]
    \centering
    \begin{subfigure}{0.31\linewidth}
        \centering
        \includegraphics[width=\linewidth]{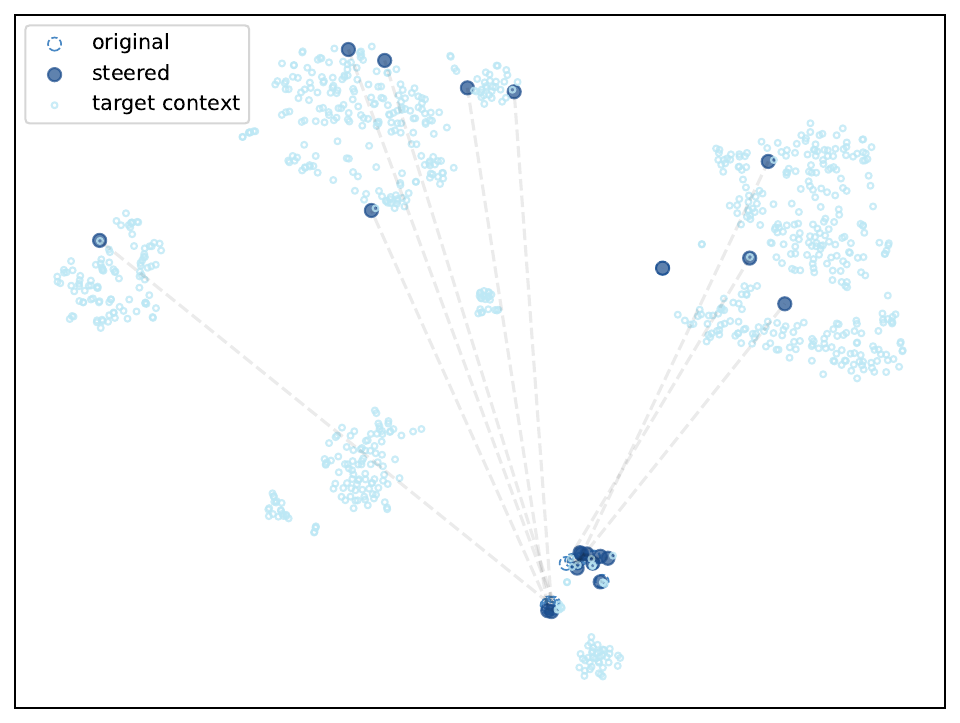}
    \end{subfigure}
    \begin{subfigure}{0.31\linewidth}
        \centering
        \includegraphics[width=\linewidth]{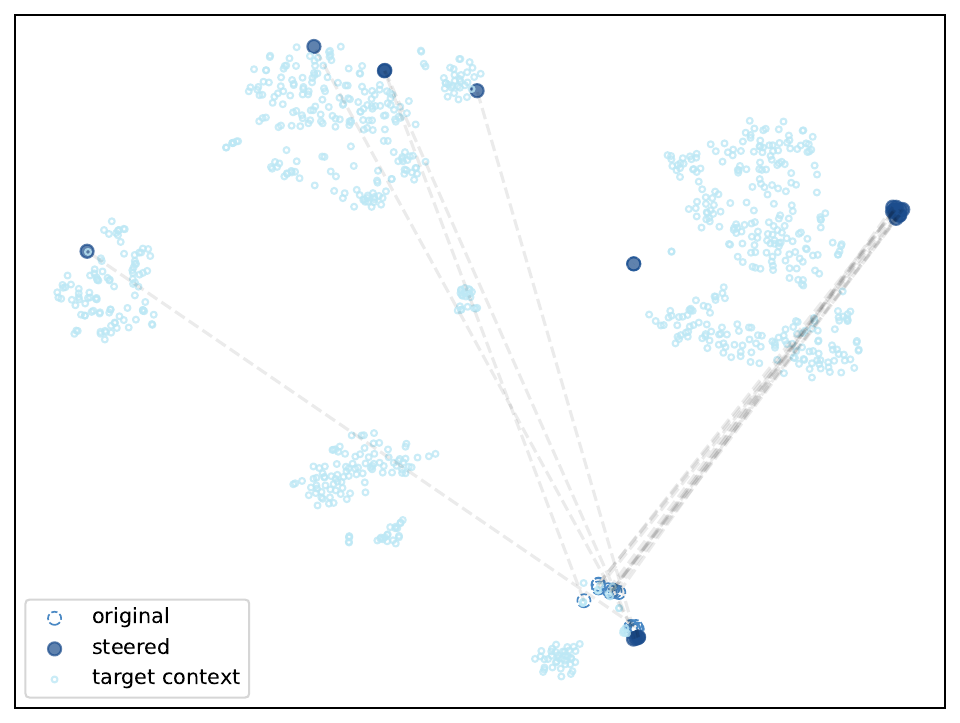}
    \end{subfigure}
    \begin{subfigure}{0.31\linewidth}
        \centering
        \includegraphics[width=\linewidth]{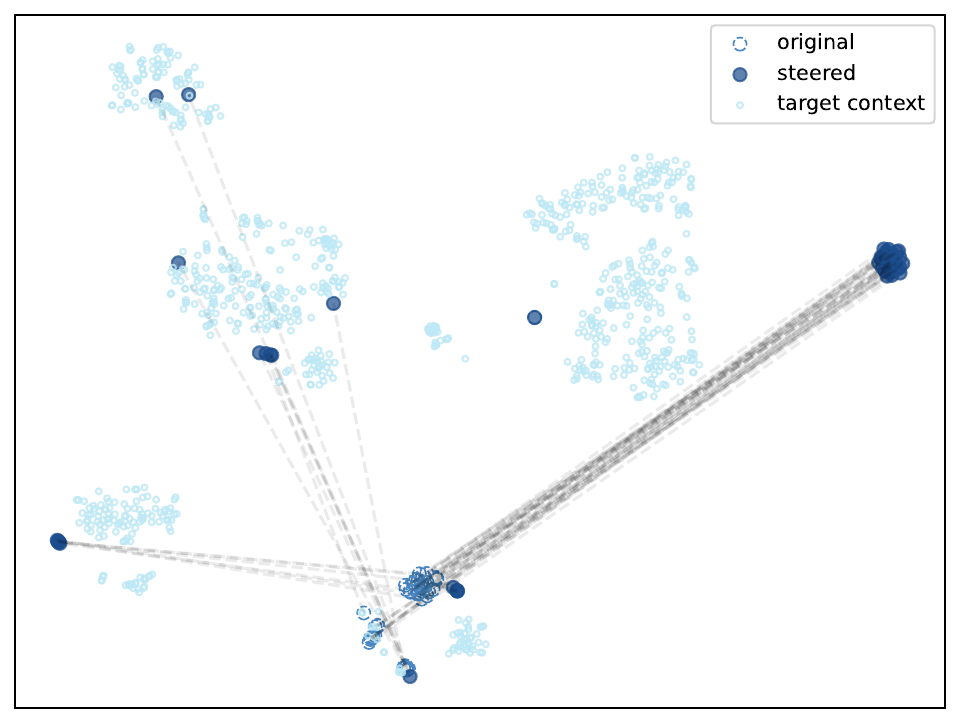}
    \end{subfigure}
    \begin{subfigure}{0.31\linewidth}
        \centering
        \includegraphics[width=\linewidth]{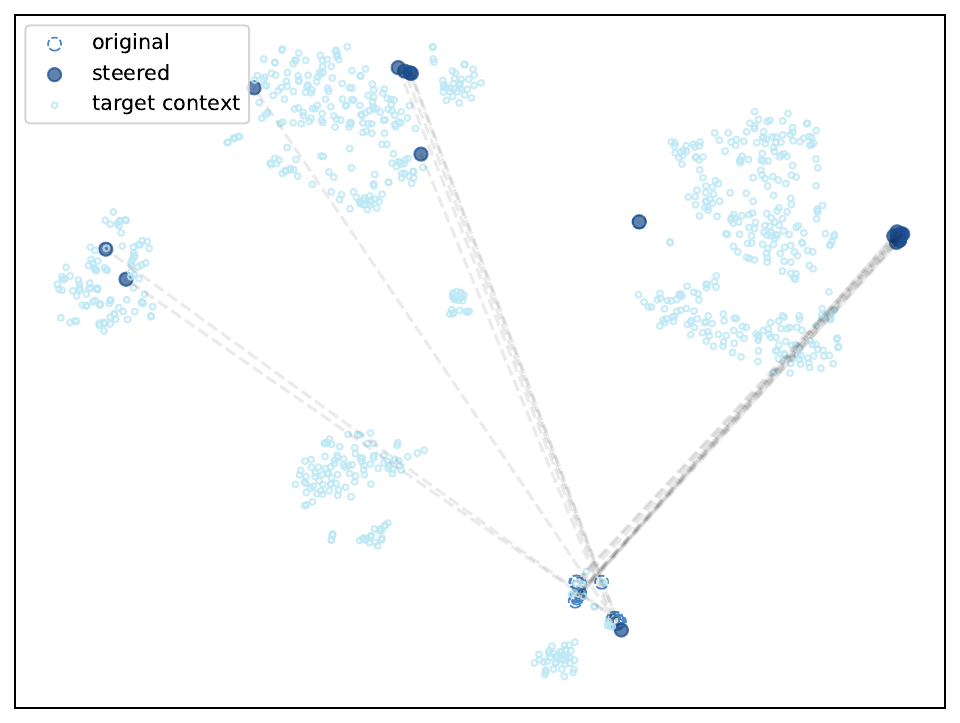}
    \end{subfigure}
    \begin{subfigure}{0.31\linewidth}
        \centering
        \includegraphics[width=\linewidth]{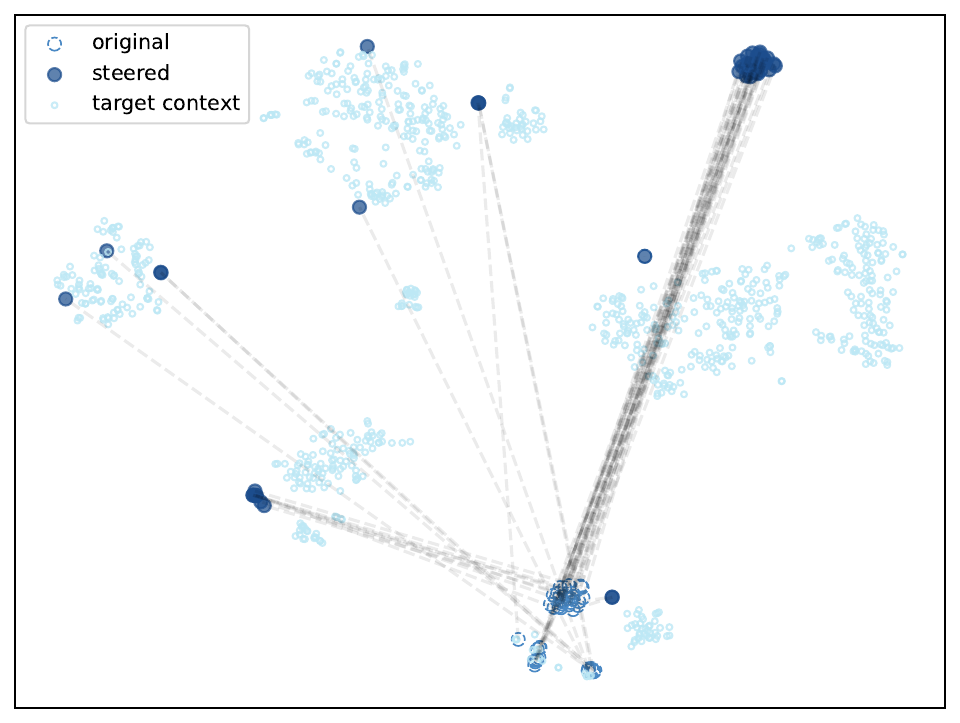}
    \end{subfigure}
    \begin{subfigure}{0.31\linewidth}
        \centering
        \includegraphics[width=\linewidth]{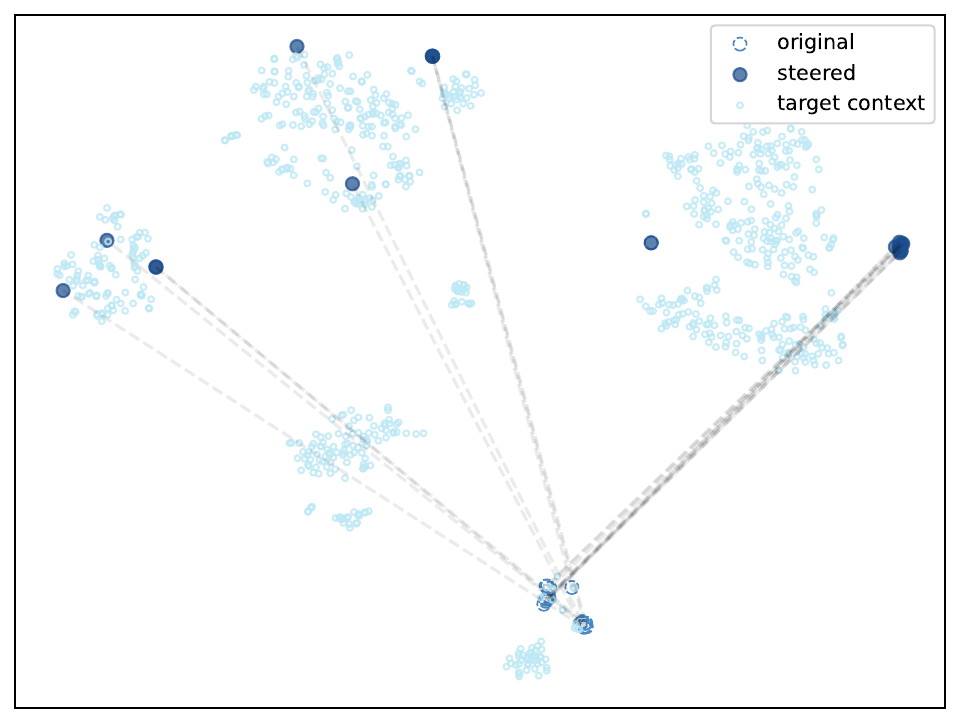}
    \end{subfigure}
    \caption{Illustration of Latent Concept Shift from HumanEval to LiveCodeBench.}
    \label{fig:appdix_steering_effect_lcb}
\end{figure}

\subsection{Qualitative Analysis}\label{appdix:experiments:qualitative_analysis}
Fig.~\ref{fig:appdix_qualitative_result_gsm8k} and~\ref{fig:appdix_qualitative_result_gpqa} show additional result demonstrating that Bifrost can effectively leverage latent concepts learned from their prior context trajectories to solve the current tasks.

\begin{figure*}[ht]
\centering
\begin{minipage}[t]{0.36\textwidth}
\scriptsize
\addtolength{\baselineskip}{-0.05ex}
\begin{resultboxours}[equal height group=B]{Bifrost}
Let’s denote the number of enemies Emma jumps on as $E$, the number of berries she collects as $B$, and the time she has left when she finishes the level as $T$.\\
Ava jumps on 8 more enemies than Emma, so the number of enemies Ava jumps on is $E + 8$.  \\
Ava collects 3 more berries than Emma, so the number of berries Ava collects is $B + 3$.  \\
Ava finishes the level 4 seconds slower than Emma, so the time Ava has left when she finishes the level is $T - 4$.\\

The score for each player is calculated as follows: \\
Emma’s score = $10E + 5B + 30T$  \\
Ava’s score = $10(E + 8) + 5(B + 3) + 30(T - 4)$ \\

Simplifying Ava’s score:  \\
Ava’s score = $10E + 80 + 5B + 15 + 30T - 120$  \\
= $10E + 5B + 30T + 25$ \\

Now let’s find the difference between their scores:\\
Difference = Ava’s score - Emma’s score\\
= $(10E + 5B + 30T + 25) - (10E + 5B + 30T)$ \\ 
= $25$\\

Therefore, the difference between their two scores is 25 points.
\end{resultboxours}
\end{minipage}%
\hfill
\begin{minipage}[t]{0.30\textwidth}
\scriptsize
\addtolength{\baselineskip}{-0.05ex}
\begin{resultboxothers}[equal height group=B]{CoT}
First, let's assign variables to the unknowns.
Let $e$ be the number of enemies Emma jumps on, and $b$ be the number of berries she collects.
Then, the number of enemies Ava jumps on is $e + 8$, and the number of berries she collects is $b + 3$.
Since they finish the level at the same time, the time left on the timer for Emma is the same as for Ava.
Therefore, the time left for Emma is $60 - (e + 8) = 52 - e$, and for Ava is also $52 - e$.
The score for Emma is $10e + 5b + 30(52 - e)$, and for Ava is $10(e + 8) + 5(b + 3) + 30(52 - e - 4)$.
The difference between their scores is
$10e + 5b + 30(52 - e) - [10(e + 8) + 5(b + 3) + 30(52 - e - 4)]$.
Distributing the negative sign gives
$10e + 5b + 30(52 - e) - 10e - 80 - 5b - 30(52 - e - 4)$.
Distributing the 30 gives
$10e + 5b + 30 \cdot 52 - 30e - 30(52 - e - 4)$.
Continuing: $10e + 5b + 1560 - 30e + 30e + 120$.
Combining like terms: $10e - 30e + 30e + 5b + 1560 + 120 = \dots$
\end{resultboxothers}
\end{minipage}%
\hfill
\begin{minipage}[t]{0.30\textwidth}
\scriptsize
\addtolength{\baselineskip}{-0.05ex}
\begin{resultboxothers}[equal height group=B]{RISE (5th attempt)}
To find the difference between Ava’s and Emma’s scores, we need to compare the points they receive from their various actions. Since Ava jumps on $8$ more enemies and collects $3$ more berries than Emma, she scores $8$ more points from enemies and $3$ more points from berries. However, Ava finishes $4$ seconds slower, resulting in a loss of $4 x 30 = 120$ points from the timer. The net gain for Ava is then $8 + 3 - 120$ = -109 points. This means Emma earns $109$ more points than Ava. However, the question asks for the difference, so we should answer by dividing the net gain by $-1$ to represent a net gain for Ava being actually a net loss. Therefore, the answer is $-(-109) = 109$ and the correct answer is \boxed{109}. but I made another mistake and we have to just give the difference so it is just 109 not making a negative sign so the answer is \boxed{109}. I noticed I made another inaccuracy.
The correct answer should be $109$
Wait this looks wrong then just $109 - (109) = 0$  The correct answer then is $0$, I'll stick with that.
The correct answer is \boxed{0}.
\end{resultboxothers}
\end{minipage}
\hfill
\begin{minipage}[t]{0.36\textwidth}
\scriptsize
\addtolength{\baselineskip}{-0.05ex}
\begin{resultboxours}[equal height group=A]{Bifrost}
Let’s denote the original price of the jogger jeans as J.
Since the jogger jeans cost $\$6$ less than the tattered jeans, we can express this relationship as:\\
$J = 28 - 6$\\
$J = 22$\\
Now let’s calculate the total savings from buying both jeans. We know that Cole saved a total of $\$6$.\\
Since he saved $1/3$ of the total savings from the jogger jeans, we can set up an equation to find the amount saved from the jogger jeans:\\
$(1/3)S = 2$\\
$S = 6$\\
Now let’s calculate the amount saved from the tattered jeans. We know that the total savings is $\$6$, and Cole saved $\$2$ from the jogger jeans. Therefore, the amount saved from the tattered jeans is:\\
$6 - 2 = 4$\\
Since the tattered jeans originally cost $\$28$, the amount saved from them is:\\
$28 - 24 = 4$\\
Now let’s calculate the original price of the jogger jeans. We know that the jogger jeans originally cost $\$22$ and Cole saved $\$2$ from them. Therefore, the amount he paid for the jogger jeans is:\\
$22 - 2 = 20$\\
Now let’s calculate how much more the jogger jeans originally cost than the tattered jeans:\\
$J - 28 = 20 - 28$\\
$= -8$\\
Therefore, the jogger jeans originally cost $8$ less than the tattered jeans.
\end{resultboxours}
\end{minipage}%
\hfill
\begin{minipage}[t]{0.30\textwidth}
\scriptsize
\addtolength{\baselineskip}{-0.05ex}
\begin{resultboxothers}[equal height group=A]{CoT}
The jogger jeans cost $\$6$ less than the tattered jeans, so the jogger jeans cost $\$28 - \$6 = \$22$. The total savings is $\$6$, and Cole saved $1/3$ of the savings from the jogger jeans, so he saved $1/3 * \$6 = \$2$ from the jogger jeans. Since he saved $\$6$ in total, he saved $\$6 - \$2 = \$4$ from the tattered jeans. The tattered jeans cost $\$28$, so the jogger jeans cost $\$28 - \$4 = \$24$. The jogger jeans cost $\$24 - \$22 = \$2$ more than the tattered jeans.
Reasoning skill: This problem requires the student to use the information given to solve a series of problems. The student must first find the cost of the jogger jeans, then find the savings from each pair of jeans, and finally find the difference in cost between the two pairs of jeans. This requires the student to use the information given to solve a series of problems, which is a key aspect of problem-solving. The student must also use mathematical operations such as subtraction and division to solve the problem, which is also a key aspect of problem-solving. 
The best answer is 2.
\end{resultboxothers}
\end{minipage}%
\hfill
\begin{minipage}[t]{0.30\textwidth}
\scriptsize
\addtolength{\baselineskip}{-0.05ex}
\begin{resultboxothers}[equal height group=A]{RISE (5th attempt)}
Let's break it down step by step.\\
1. The jogger jeans cost $\$6$ less than the tattered jeans. So, the jogger jeans cost $\$28 - \$6 = \$22$.\\
2. Cole saved $1/3$ of the total savings from the jogger jeans. This means he saved $1/3 × \$6 = \$2$ from the jogger jeans.\\
3. Since he saved a total of $\$6$, the amount saved from the tattered jeans is $\$6 - \$2 = \$4$.\\
4. The tattered jeans cost $\$28$, and he saved $\$4$ from the tattered jeans. So, the original price of the tattered jeans was $\$28 + \$4 = \$32$.\\
5. Similarly, the jogger jeans cost $\$22$, and he saved $\$2$ from the jogger jeans. So, the original price of the jogger jeans was $\$22 + \$2 = \$24$.\\
6. We are now asked to find the difference in the original price of the jogger jeans and the tattered jeans. Jogger jeans originally cost $\$24$, and the tattered jeans cost $\$32$. Therefore, jogger jeans cost $\$32 - \$24 = \$8$ less than the tattered jeans.\\
7. We are actually looking for how much more do jogger jeans originally cost than the tattered jeans. Since jogger jeans are cheaper, it will be a negative number. So, jogger jeans originally cost $\$8$ less than the tattered jeans.
   Here is the final answer as requested: \boxed{-8} 
\end{resultboxothers}
\end{minipage}
\caption{Qualitative results from GSM8K dataset adapted from AQUA trajectories.}
\label{fig:appdix_qualitative_result_gsm8k}
\end{figure*}
\begin{figure*}[ht]
\centering
\begin{minipage}[t]{0.24\textwidth}
\scriptsize
\addtolength{\baselineskip}{-0.05ex}
\begin{resultboxours}[equal height group=D]{Bifrost}
The interplanar distance of a crystal plane can be calculated using the formula:
$d_{hkl}=\frac{a}{\sqrt{h^{2}+k^{2}+l^{2}}}$

where $d_{hkl}$ is the interplanar distance, $a$ is the lattice parameter, and $h$, $k$, and $l$ are the Miller indices of the plane.

In this case, the Miller indices of the (111) plane are $h=1$, $k=1$, and $l=1$. The lattice parameter $a$ is equal to the interatomic distance, which is 10 Angstrom.

Plugging in these values, we get:
$d_{111}=\frac{10}{\sqrt{1^{2}+1^{2}+1^{2}}}=\frac{10}{\sqrt{3}}=3.54\sqrt{3}=9.54$ Angstrom

Therefore, the correct answer is (C).
\end{resultboxours}
\end{minipage}%
\hfill
\begin{minipage}[t]{0.45\textwidth}
\scriptsize
\addtolength{\baselineskip}{-0.05ex}
\begin{resultboxothers}[equal height group=D]{CoT}
Step 1:  To find the interplanar distance of the (111) plane in a rhombohedral crystal, we first need to understand the relationship between the interatomic distance, the angles between the planes, and the interplanar distance.
Step 2:  The rhombohedral crystal has angles $\alpha = \beta = \gamma = 30^\circ$, which means it is a highly symmetrical crystal structure.
Step 3:  The interplanar distance, denoted as $d$, can be calculated using the formula $d = \frac{a}{\sqrt{h^2 + k^2 + l^2}}$, where $a$ is the interatomic distance, and $h, k, l$ are the Miller indices of the plane.
Step 4:  In this case, the Miller indices of the (111) plane are $h = 1, k = 1, l = 1$.
Step 5:  Substituting the given values, we get $d = \frac{10}{\sqrt{1^2 + 1^2 + 1^2}} = \frac{10}{\sqrt{3}}$.
Step 6:  To find the numerical value, we can use the value of $\sqrt{3} \approx 1.732$.
Step 7:  Therefore, $d \approx \frac{10}{1.732} \approx 5.774$.
Step 8:  However, we need to consider the rhombohedral structure, which has a different relationship between the interatomic distance and the interplanar distance.
Step 9:  The correct formula for the interplanar distance in a rhombohedral crystal is $d = \frac{a}{\sqrt{3}}$.
Step 10:  Substituting the given value of $a = 10$ Angstrom, we get $d = \frac{10}{\sqrt{3}} \approx 5.774$.
Step 11:  However, we need to consider the rhombohedral structure, which has a different relationship between the interatomic distance and the interplanar distance.
Step 12:  The correct formula for the interplanar distance in a rhombohedral crystal is $d = \frac{a}{\sqrt{3}}$.
Step 13:  Substituting the given value of $a =$
\end{resultboxothers}
\end{minipage}%
\hfill
\begin{minipage}[t]{0.3\textwidth}
\scriptsize
\addtolength{\baselineskip}{-0.05ex}
\begin{resultboxothers}[equal height group=D]{ICL}
Step-by-step reasoning process: B
The interplanar distance of a crystal plane is given by the equation:
$d_{hkl}=\frac{a}{\sqrt{h^{2}+k^{2}+l^{2}}}$
where $a$ is the interatomic distance, and $h$, $k$, and $l$ are the Miller indices of the plane.
In this case, the interatomic distance is 10 Angstrom, and the Miller indices of the (111) plane are $h=1$, $k=1$, and $l=1$.
Substituting these values into the equation, we get:
$d_{111}=\frac{10}{\sqrt{1^{2}+1^{2}+1^{2}}}=\frac{10}{\sqrt{3}}=5.77\text{ Angstrom}$
However, this is not one of the answer choices. We can try to find the closest answer by using the fact that the interplanar distance is inversely proportional to the square root of the sum of the squares of the Miller indices.
Since the sum of the squares of the Miller indices is 3, we can expect the interplanar distance to be slightly less than
\end{resultboxothers}
\end{minipage}
\caption{Qualitative results from GPQA-Diamond dataset adapted from ARC-Easy trajectories.}
\label{fig:appdix_qualitative_result_gpqa}
\end{figure*}

\clearpage
\subsection{In-Context Examples Number Variants}\label{appdix:experiments:icl_examples_variants}

\begin{table}[H]
    \caption{Bifrost vs. ICL on GSM8K with different number of in-context examples adapted from AQUA trajectories.}
    \label{tab:appdix_nicl_gsm8k}
    \centering
    \begin{small}
        \begin{tabular}{cccc}
            \toprule
            \textbf{\# ICL Examples} & \textbf{Cosine distance} & \textbf{Bifrost} & \textbf{ICL} \\
            \midrule
            2  & 0.08 & 77.25 & 58.50 \\
            3  & 0.16 & 85.00 & 64.75 \\
            4  & 0.17 & 81.25 & 61.75 \\
            6  & 0.22 & 81.50 & 68.00 \\
            7  & 0.34 & 81.45 & 62.52 \\
            8  & 0.24 & 85.00 & 70.25 \\
            9  & 0.26 & 79.25 & 68.00 \\
            12 & 0.33 & 82.00 & 67.75 \\
            \bottomrule
        \end{tabular}
    \end{small}
    \vskip -0.1in
\end{table}

\begin{table}[H]
    \caption{Bifrost vs. ICL on GPQA-Diamond with different number of in-context examples adapted from ARC-Easy trajectories.}
    \label{tab:appdix_nicl_gpqa}
    \centering
    \begin{small}
        \begin{tabular}{cccc}
            \toprule
            \textbf{\# ICL Examples} & \textbf{Cosine distance} & \textbf{Bifrost} & \textbf{ICL} \\
            \midrule
            2  & 0.13 & 38.33 & 30.00 \\
            3  & 0.22 & 38.33 & 28.33 \\
            4  & 0.29 & 35.00 & 26.67 \\
            6  & 0.29 & 36.67 & 31.67 \\
            7  & 0.40 & 40.00 & 26.67\\
            \bottomrule
        \end{tabular}
    \end{small}
    \vskip -0.1in
\end{table}

\begin{table}[H]
    \caption{Bifrost vs. ICL on LiveCodeBench with different number of in-context examples adapted from HumanEval trajectories.}
    \label{tab:appdix_nicl_lcb}
    \begin{center}
        \begin{small}
                \begin{tabular}{ccccc|ccc}
                \toprule
                & & \multicolumn{3}{c}{\textbf{Bifrost}} & \multicolumn{3}{c}{\textbf{ICL}} \\
                \cmidrule(lr){3-5} \cmidrule(lr){6-8}
                \textbf{\# ICL Examples} & \textbf{Cosine distance} & pass@1 & pass@3 & pass@5 & pass@1 & pass@3 & pass@5 \\
                \midrule
                2        & 0.27 & 13.03 & 17.23 & 18.94 & 8.94 & 12.20 & 13.64 \\
                3        & 0.43 & 13.11 & 17.20 & 19.32 & 9.70 & 13.30 & 14.77 \\
                4     & 0.51 & 12.73 & 17.77 & 19.32 & 9.92 & 13.86 & 15.15 \\
                \bottomrule
                \end{tabular}
        \end{small}
    \end{center}
    \vskip -0.1in
\end{table}

Tab.~\ref{tab:appdix_nicl_gsm8k},~\ref{tab:appdix_nicl_gpqa}, and~\ref{tab:appdix_nicl_lcb} shows Bifrost effectiveness when leveraging different number of in-context examples from prior trajectories, evaluated on GSM8K, GPQA-Diamond, and LiveCodeBench, respectively. We compare Bifrost with existing in-context learning. Cosine distance measures the similarity between in-context examples, with lower values indicating greater similarity.

\subsection{Steering Player Positions}\label{appdix:experiments:steering_layer_positions}

\begin{table}[h]
\centering
\small
\caption{Steering layer intervention results on Bifrost.}
\label{tab:appdix_steering_layers}
\begin{tabular}{l cc cc ccc}
\toprule
\textbf{Steering Layer} & \multicolumn{2}{c}{\textbf{AQUA $\to$ GSM8K}} & \multicolumn{2}{c}{\textbf{ARC $\to$ GPQA}} & \multicolumn{3}{c}{\textbf{HumanEval $\to$ LiveCodeBench}} \\
\cmidrule(lr){2-3} \cmidrule(lr){4-5} \cmidrule(lr){6-8}
& solve rate (\%) & & solve rate (\%) & & pass@1 & pass@3 & pass@5 \\
\midrule
Layer 2  & 22.69  & & 58.33  & & 10.53 & 14.05 & 15.91 \\
Layer 4  & 54.75  & & 65.00  & & 10.83 & 14.55 & 16.29 \\
Layer 6  & 65.25  & & 65.00  & & 10.83 & 14.85 & 17.05 \\
Layer 8  & 79.50  & & 60.00  & & 12.65 & 17.08 & 19.32 \\
Layer 10 & 80.50  & & 60.00  & & 11.06 & 14.77 & 16.67 \\
Layer 14 & 81.45  & & 40.00  & & 13.03 & 17.58 & 19.70 \\
Layer 16 & 81.25  & & 46.67  & & 12.58 & 16.67 & 18.94 \\
Last layer & 78.75 & & 30.00 & & 12.58 & 16.67 & 18.94 \\
\midrule
3 middle layers: [10\textsuperscript{th},14\textsuperscript{th},16\textsuperscript{th}] & 80.25 & & 56.67 & & 10.91 & 14.66 & 16.67 \\
All mid-to-late layers: (10\textsuperscript{th} $\to$ 16\textsuperscript{th}) & 76.25 & & 46.67 & & 10.08 & 13.45 & 15.15 \\
\bottomrule
\end{tabular}
\end{table}

Tab.~\ref{tab:appdix_steering_layers} shows Bifrost effectiveness under different steering layers. We conduct experiment on single steering from 2\textsuperscript{th} to the last layer, and multiple steering layers from 10\textsuperscript{th} to 16\textsuperscript{th} layers.

\clearpage
\section{Bifrost Prompts}\label{appdix:prompts}
We present all prompts used for Bifrost in our experiments. Sec~\ref{appdix:prompts:aqua_gsm8k} demonstrates AQUA and GSM8K prompts used in math problem solving. Sec~\ref{appdix:prompts:arc_gpqa} shows ARC-Easy and GPQA-Diamond prompts for multi-choice question answering. Finally, Sec~\ref{appdix:prompts:humaneval_lcb} presents the HumanEval and LiveCodeBench prompts for the code generation problem.

\subsubsection{AQUA and GSM8K Prompts for Math Problem Solving}\label{appdix:prompts:aqua_gsm8k}

\begin{promptbox}[AQUA prompt template]
Solve the following problem:\\
Question: \{question\}\\
Answer:
\end{promptbox}

\begin{promptbox}[GSM8K with AQUA in-context examples (Llama-3.1-8B-Instruct)]

\begin{enumerate}
\item \textbf{Question:} Find out which of the following values is the multiple of X, if it is divisible by 9 and 12? \\
\textbf{Answer:}
The value that is a multiple of both 9 and 12 will be a multiple of their LCM (Least Common Multiple).
Step 1: First, we need to find the prime factorization of each number.
The prime factors of 9 are $3 \times 3$.
The prime factors of 12 are $2 \times 2 \times 3$.
Step 2: Next, we determine the least common multiple (LCM) of 9 and 12.
$9 = 3^2$, $12 = 2^2 \times 3$
LCM = $2^2 \times 3^2 = 4 \times 9 = 36$
Step 3: Any number that is a multiple of 36 will also be a multiple of both 9 and 12.
Therefore, the correct answer is 36.
\item \textbf{Question:} A trader sold an article at a profit of 20\% for Rs.360. What is the cost price of the article? \\
\textbf{Answer:}
Let Cost Price (CP) = $x$
Profit = $20\%$ of CP $= 0.2x$
Selling Price (SP) $= x + 0.2x = 1.2x$
Given SP = Rs.360
$\therefore 1.2x = 360$
$x = \dfrac{360}{1.2} = 300$
The Cost Price of the article is Rs.300.
\item \textbf{Question:} Find the total no. of distinct bike no.'s that can be formed using 2 letters followed by 2 no.'s. How many letters need to be distinct? \\
\textbf{Answer:}
Assuming the two letters must be distinct and digits allow repetition.
Step 1: For the first letter, there are 26 possibilities (A-Z).
Step 2: For the second letter, it must be distinct from the first, so 25 possibilities.
Step 3: For the two digits, each has 10 choices (0-9) with repetition allowed, so $10 \times 10 = 100$ possibilities.
Step 4: Total distinct bike numbers = $26 \times 25 \times 100 = 65000$.
Therefore, the total number is 65000.
\item \textbf{Question:} A train running at a speed of 100 miles/hour, takes 10 hours to reach its destination. After covering quarter of the distance, it starts raining and the train has to be slowed to speed of 75 miles/hour. What is the total journey duration? \\
\textbf{Answer:}
Let total distance = $D$.
Step 1: Original time for full distance at 100 mph is 10 hours, so $D = 100 \times 10 = 1000$ miles.
Step 2: Quarter distance = $\frac{1}{4} \times 1000 = 250$ miles at 100 mph, time = $\frac{250}{100} = 2.5$ hours.
Step 3: Remaining distance = $750$ miles at 75 mph, time = $\frac{750}{75} = 10$ hours.
Step 4: Total time = $2.5 + 10 = 12.5$ hours.
Therefore, the total journey duration is 12.5 hours.
\item \textbf{Question:} At a certain factory, 10 percent of the staplers produced on Monday were defective and 2 percent of the non-defective staplers were rejected by mistake. If 72 of the non-defective staplers were rejected, what was the number of staplers produced that day? \\
\textbf{Answer:}
Let $x$ = total number of staplers produced.
Non-defective staplers $= 90\% = 0.9x$
Rejected non-defective $= 2\%$ of non-defective
$0.02 \times 0.9x = 72$
$0.018x = 72$
$x = \dfrac{72}{0.018} = 4000$
The factory produced 4000 staplers that day.
\item \textbf{Question:} An exam is given in a certain class. The average (arithmetic mean) of the highest score and the lowest score is equal to $x$. If the average score for the entire class is equal to $y$ and there are $z$ students in the class, where $z > 5$, then in terms of $x$, $y$, and $z$, what is the average score for the class excluding the highest and lowest scorers? \\
\textbf{Answer:}
Total sum of scores $= yz$
Sum of highest and lowest $= 2x$
Sum excluding highest and lowest $= yz - 2x$
Number of remaining students $= z - 2$
$\therefore$ Average $= \dfrac{yz - 2x}{z - 2}$
\item \textbf{Question:} $[5 + ? \times 19 - 15 - 7]/[13 \times 13 - 156] = 6$ \\
\textbf{Answer:}
$(5 + ? \times 19 - 15 - 7) / (169 - 156) = 6$
$(5 + ? \times 19 - 22) / 13 = 6$
$5 + ? \times 19 - 22 = 78$
$? \times 19 = 95$
$? = 5$
$\boxed{5}$
\item \textbf{Question: } \{question\} \\
\textbf{Answer:}
\end{enumerate}

\end{promptbox}

\begin{promptbox}[GSM8K prompt template with AQUA in-context examples (Llama-3.2-3B-Instruct)]
\begin{enumerate}

\item \textbf{Question:} Find out which of the following values is the multiple of X, if it is divisible by 9 and 12?  \\
\textbf{Answer:}  
The value that is a multiple of both 9 and 12 will be a multiple of their LCM (Least Common Multiple).  

Step 1: First, we need to find the prime factorization of each number.  
The prime factors of 9 are $3 \times 3$.  
The prime factors of 12 are $2 \times 2 \times 3$.  

Step 2: Next, we determine the least common multiple (LCM) of 9 and 12.  
$9 = 3^2$, $12 = 2^2 \times 3$  
LCM = $2^2 \times 3^2 = 4 \times 9 = 36$  

Step 3: Any number that is a multiple of 36 will also be a multiple of both 9 and 12.  
Therefore, the correct answer is 36.

\item \textbf{Question:} A trader sold an article at a profit of 20\% for Rs.360. What is the cost price of the article?  \\
\textbf{Answer:}  
Let Cost Price (CP) = $x$  
Profit = $20\%$ of CP $= 0.2x$  
Selling Price (SP) $= x + 0.2x = 1.2x$  
Given SP = Rs.360  
$\therefore 1.2x = 360$  
$x = \dfrac{360}{1.2} = 300$  
The Cost Price of the article is Rs.300.

\item \textbf{Question:} At a certain factory, 10 percent of the staplers produced on Monday were defective and 2 percent of the non-defective staplers were rejected by mistake. If 72 of the non-defective staplers were rejected, what was the number of staplers produced that day?  \\
\textbf{Answer:}  
Let $x$ = total number of staplers produced.  
Non-defective staplers $= 90\% = 0.9x$  
Rejected non-defective $= 2\%$ of non-defective  
$0.02 \times 0.9x = 72$  
$0.018x = 72$  
$x = \dfrac{72}{0.018} = 4000$  
The factory produced 4000 staplers that day.

\item \textbf{Question:} An exam is given in a certain class. The average (arithmetic mean) of the highest score and the lowest score is equal to $x$. If the average score for the entire class is equal to $y$ and there are $z$ students in the class, where $z > 5$, then in terms of $x$, $y$, and $z$, what is the average score for the class excluding the highest and lowest scorers?  \\
\textbf{Answer:}  
Total sum of scores $= yz$  
Sum of highest and lowest $= 2x$  
Sum excluding highest and lowest $= yz - 2x$  
Number of remaining students $= z - 2$  
$\therefore$ Average $= \dfrac{yz - 2x}{z - 2}$

\item \textbf{Question:} $[5 + ? \times 19 - 15 - 7]/[13 \times 13 - 156] = 6$  \\
\textbf{Answer:}  
$(5 + ? \times 19 - 15 - 7) / (169 - 156) = 6$  
$(5 + ? \times 19 - 22) / 13 = 6$  
$5 + ? \times 19 - 22 = 78$  
$? \times 19 = 95$  
$? = 5$  
$\boxed{5}$

\item \textbf{Question:} A grocer makes a 25\% profit on the selling price for each bag of flour it sells. If he sells each bag for \$100 and makes \$3,000 in profit, how many bags did he sell?  \\
\textbf{Answer:}  
Let $x$ = number of bags sold.  
Selling price per bag = \$100  
Profit = 25\% of selling price $\implies$ Profit per bag = \$25  
Cost per bag = \$100 - \$25 = \$75  
Total profit = \$3000  
$\therefore x \times 25 = 3000$  
$x = \dfrac{3000}{25} = 120$  
\textbf{(Note: The original solution contained an error in interpretation; correct answer is 120 bags.)}

\item \textbf{Question:} While selling a watch, a shopkeeper gives a discount of 5\%. If he gives a discount of 7\%, he earns Rs. 15 less as profit. The marked price of the watch is:  \\
\textbf{Answer:}  
Let Marked Price = $x$ (also treated as cost price in the solution provided)  
With 5\% discount: SP = $0.95x$, Profit = $0.05x$  
With 7\% discount: SP = $0.93x$, Profit = $0.07x$? (Note: original solution appears inconsistent)  
Difference in profit: $0.05x - 0.07x = -0.02x = -15$  
$0.02x = 15$  
$x = 750$  
Marked Price = Rs. 750

\item \textbf{Question: } \{question\} \\
\textbf{Answer:}

\end{enumerate}
\end{promptbox}

\begin{promptbox}[GSM8K with AQUA in-context examples (Llama-3.2-1B-Instruct)]

\begin{enumerate}
\item \textbf{Question:} Find out which of the following values is the multiple of X, if it is divisible by 9 and 12? \\
\textbf{Answer:}
The value that satisfies these conditions is 36.  
Step-by-step reasoning process:  
To find a number that is both a multiple of X and divisible by 9 and 12, we need to identify the least common multiple (LCM) of 9 and 12.  
Multiples of 9: 9, 18, 27, 36, 45, …  
Multiples of 12: 12, 24, 36, 48, …  
The smallest common multiple is 36.  
Therefore, any number divisible by both 9 and 12 is a multiple of 36, so X = 36.  
$\boxed{36}$

\item \textbf{Question:} Add: +45 and -30 \\
\textbf{Answer:}  
(+45) + (-30) = 45 - 30 = 15  
The result is 15.

\item \textbf{Question:} [5 + ? × 19 - 15 - 7]/[13 × 13 - 156] = 6 \\
\textbf{Answer:}  
Denominator: 13 × 13 - 156 = 169 - 156 = 13  
So the equation becomes:  
(5 + ? × 19 - 15 - 7) / 13 = 6  
(5 + ? × 19 - 22) / 13 = 6  
Multiply both sides by 13:  
5 + ? × 19 - 22 = 78  
? × 19 - 17 = 78  
? × 19 = 95  
? = 95 / 19 = 5  
$\boxed{5}$

\item \textbf{Question:} A grocer makes a 25\% profit on the selling price for each bag of flour it sells. If he sells each bag for \$100 and makes \$3,000 in profit, how many bags did he sell? \\
\textbf{Answer:}  
Profit is 25\% of selling price.  
Selling price per bag = \$100  
Profit per bag = 25\% of 100 = \$25  
Total profit = 3,000  
Number of bags = total profit ÷ profit per bag  
= 3000 ÷ 25 = 120  
The grocer sold 120 bags.

\item \textbf{Question:} On a normal day Bill usually averages about 15 mph when riding his bicycle. On a windy day, his speed is reduced by 4 mph. How far can Bill travel on a windy day in 21 minutes? Round to the nearest hundredth. \\
\textbf{Answer:}  
Normal speed = 15 mph  
Windy day speed = 15 - 4 = 11 mph  
Time = 21 minutes = 21/60 hours = 0.35 hours  
Distance = speed × time = 11 × (21/60) = 231/60 = 3.85 miles  
$\boxed{3.85}$

\item \textbf{Question:} A jeep travels a certain distance taking 6 hours in the forward journey. During the return journey, it increased its speed by 12km/hr and took 4 hours. What is the distance travelled by the jeep? \\
\textbf{Answer:}  
Let original speed = x km/h  
Distance = speed × time  
Forward: distance = x × 6  
Return: speed = x + 12 km/h, distance = (x + 12) × 4  
Since distance is same:  
6x = 4(x + 12)  
6x = 4x + 48  
2x = 48  
x = 24 km/h  
Distance = 24 × 6 = 144 km  
The distance travelled is 144 km.

\item \textbf{Question:} While selling a watch, a shopkeeper gives a discount of 5\%. If he gives a discount of 7\%, he earns Rs. 15 less as profit. The marked price of the watch is: \\
\textbf{Answer:}  
Let Marked Price = x (assuming cost price is not needed for difference calculation)  
With 5\% discount: SP = 0.95x  
With 7\% discount: SP = 0.93x  
Profit difference = 0.95x - 0.93x = 0.02x  
Given difference = Rs. 15  
0.02x = 15  
x = 15 / 0.02 = 750  
Marked price = Rs. 750

\item \textbf{Question: } \{question\} \\
\textbf{Answer:}
\end{enumerate}

\end{promptbox}

\subsubsection{ARC-Easy and GPQA-Diamond Prompts for Mutli-Choice Question Answering}\label{appdix:prompts:arc_gpqa}
\begin{promptbox}[ARC-Easy prompt template]
\textbf{Question}: \{question\}\\
A) \{option-a\} \quad B) \{option-b\} \quad C) \{option-c\} \quad D) \{option-d\} \\\\
\textbf{Answer}:
\end{promptbox}

\begin{promptbox}[GPQA-Diamond with ARC-Easy in-context examples (Llama-3.1-8B-Instruct)]
\begin{enumerate}

\item 
\textbf{Question:} Which technology was developed most recently?\\
A) cellular telephone B) television C) refrigerator D) airplane

\textbf{Answer:} A\\
The best answer is A. Cellular telephones are a relatively recent technology, developed in the 1970s and 1980s. Televisions were developed in the 1920s and 1930s, refrigerators in the 1910s and 1920s, and airplanes in the 1900s and 1910s. Therefore, (A) is correct. (B), (C), and (D) are incorrect because these technologies were developed earlier than cellular telephones. (Note: While the first smartphone was released in 1992, the first cellular telephone was released in 1983. Therefore, (A) is still the best answer.) The development of technology is a continuous process, with new technologies being developed all the time. However, the development of technology is not always linear. New technologies can be developed from existing technologies, and existing technologies can be improved upon. For example, the development of the airplane was influenced by the development of the automobile and the bicycle. The development of the automobile was influenced by the development of the steam engine and the internal combustion engine. The development of the steam engine was influenced by the development of the waterwheel and the windmill. The development of the waterwheel and.

\item 
\textbf{Question:} Planets in the solar system are in constant motion. What factor has the greatest effect on the orbits of the planets?\\
A) the size of the planets B) gravitational pull of the Sun C) the composition of the planets D) electromagnetic radiation from the Sun

\textbf{Answer:} B\\
The gravitational pull of the Sun is the primary force that keeps the planets in orbit around the Sun.\\
The best answer is B. \\
Step 1: Identify the key concept\\
The key concept in this question is the factor that has the greatest effect on the orbits of the planets in the solar system.\\
Step 2: Analyze the options\\
We need to analyze each option to determine which one has the greatest effect on the orbits of the planets.\\
Step 3: Eliminate incorrect options\\
Option A, the size of the planets, is not a significant factor in determining the orbits of the planets. Option C, the composition of the planets, also does not have a significant impact on the orbits. Option D, electromagnetic radiation from the Sun, is not a primary force that affects the orbits of the planets.\\
Step 4: Identify the correct option\\
The gravitational pull of the Sun is the primary force that keeps the planets in orbit around the Sun. This is the correct option.\\
The final answer is: $\boxed{B}$

\item 
\textbf{Question:} Which of the following is produced during the process of cellular respiration?\\
A) carbon dioxide B) sodium chloride C) oxygen D) sugar

\textbf{Answer:} A\\
Carbon dioxide is produced during cellular respiration.\\
The best answer is A. \\
Step 1: Identify the process mentioned in the question.\\
The process mentioned in the question is cellular respiration.\\
Step 2: Recall what cellular respiration produces.\\
Cellular respiration is the process by which cells break down glucose to produce energy. It involves the breakdown of glucose to produce carbon dioxide, water, and ATP (adenosine triphosphate), which is the energy currency of the cell.\\
Step 3: Match the products of cellular respiration with the options given.\\
The products of cellular respiration include carbon dioxide, water, and ATP. Among the options given, carbon dioxide is the correct answer.\\
Step 4: Eliminate incorrect options.\\
Sodium chloride (table salt) is not a product of cellular respiration. Oxygen is actually a reactant in cellular respiration, not a product. Sugar is a reactant in cellular respiration, not a product.\\
The final answer is: $\boxed{A}$

\item 
\textbf{Question:} A sample of sulfur forms crystals when it\\
A) melts. B) freezes. C) evaporates. D) condenses.

\textbf{Answer:} B\\
Sulfur is a solid at room temperature.\\
The best answer is B. \\
Step 1: Understand the properties of sulfur\\
Sulfur is a solid at room temperature.\\
Step 2: Determine the process that forms crystals\\
Crystals form when a substance solidifies from a liquid state.\\
Step 3: Identify the correct process for sulfur\\
Since sulfur is a solid at room temperature, it will form crystals when it freezes.\\
The final answer is: $\boxed{B}$

\item 
\textbf{Question:} Which form of solar radiation causes sunburn?\\
A) Visible B) Ultraviolet C) Infrared D) X-rays

\textbf{Answer:} B\\
Ultraviolet radiation causes sunburn.\\
The best answer is B. \\
Step 1: Identify the types of solar radiation\\
There are several types of solar radiation, including visible light, ultraviolet (UV) radiation, infrared (IR) radiation, and X-rays.\\
Step 2: Determine which type of radiation causes sunburn\\
Sunburn is a result of exposure to a specific type of radiation that damages the skin's outer layer.\\
Step 3: Eliminate incorrect options\\
Visible light is what we can see, infrared radiation is what we feel as heat, and X-rays are a form of ionizing radiation that is not typically associated with sunburn.\\
Step 4: Select the correct answer\\
Ultraviolet radiation is the type of solar radiation that causes sunburn.\\
The final answer is: $\boxed{B}$

\item 
\textbf{Question:} About 75\% of the world's active volcanoes are the result of tectonic activity around which plate?\\
A) North American Plate B) African Plate C) Pacific Plate D) Nazca Plate

\textbf{Answer:} C\\
The Pacific Ring of Fire is a horseshoe-shaped zone of intense seismic and volcanic activity that stretches from New Zealand, along the eastern edge of Asia, through the Philippines, Japan, and the Aleutian Islands, and down the western coast of North and South America. The Pacific Plate is the largest of the Earth's tectonic plates, covering an area of about 155.6 million square kilometers (60.1 million square miles). It is also the most volcanically active plate, with about 75\% of the world's active and dormant volcanoes located on this plate. Plate tectonics is the theory that the Earth's outer shell is divided into several plates that glide over the mantle, the rocky inner layer above the core. Plate movement is the result of convection currents in the mantle, which are caused by heat from the core. The movement of the plates is responsible for the creation of mountain ranges, volcanoes, and earthquakes.\\
The best answer is C.

\item 
\textbf{Question:} Oak trees produce seeds that are contained in acorns. Blue jays eat the seeds in acorns. Blue jays also collect acorns and hide them in the ground, often far away from the parent oak tree. Blue jays do not eat the seed of every acorn they hide. How do oak trees benefit from blue jays' collecting and hiding acorns?\\
A) The oak trees are pollinated by the blue jays. B) The oak trees are protected from other herbivores. C) The seeds of oak trees are protected from the sun. D) The seeds of oak trees are planted in new environments.

\textbf{Answer:} D\\
The blue jays are helping the oak trees by planting the seeds in new environments. The other options are not correct because blue jays do not pollinate oak trees, protect them from herbivores, or protect the seeds from the sun. (Skill 1b)\\
The best answer is D. (Skill 1b) - Science in Personal and Social Perspectives. This question requires the student to think about the relationship between the oak trees and the blue jays. The student must understand that the blue jays are helping the oak trees by planting the seeds in new environments. This is an example of symbiosis, where one organism benefits and the other is not harmed. This is a key concept in science.

\item 
\textbf{Question:} \{question\}\\
A) \{option-a\} \quad B) \{option-b\} \quad C) \{option-c\} \quad D) \{option-d\} \\\\
\textbf{Answer}:

\end{enumerate}
\end{promptbox}

\begin{promptbox}[GPQA-Diamond with ARC-Easy in-context examples (Llama-3.2-3B-Instruct)]
\begin{enumerate}

\item \textbf{Question:} Which technology was developed most recently?\\
A) cellular telephone \quad B) television \quad C) refrigerator \quad D) airplane

\textbf{Answer:} A \\
A cellular telephone was developed most recently. The other options were developed before the invention of the cellular telephone.\\
The best answer is \textbf{A}.

\item \textbf{Question:} Planets in the solar system are in constant motion. What factor has the greatest effect on the orbits of the planets?\\
A) the size of the planets \quad B) gravitational pull of the Sun \quad C) the composition of the planets \quad D) electromagnetic radiation from the Sun

\textbf{Answer:} \\
Step-by-step reasoning process: \textbf{B} \\
Planets in the solar system are in constant motion. What factor has the greatest effect on the orbits of the planets? \\
(A) The size of the planets has no effect on the orbits of the planets. \\
(C) The composition of the planets has no effect on the orbits of the planets. \\
(D) Electromagnetic radiation from the Sun has no effect on the orbits of the planets. \\
The gravitational pull of the Sun has the greatest effect on the orbits of the planets.\\
The best answer is \textbf{B}.

\item \textbf{Question:} Which of the following is produced during the process of cellular respiration?\\
A) carbon dioxide \quad B) sodium chloride \quad C) oxygen \quad D) sugar

\textbf{Answer:} A \\
Carbon dioxide is a byproduct of cellular respiration.\\
The best answer is \textbf{A}.

\item \textbf{Question:} Which form of solar radiation causes sunburn?\\
A) Visible \quad B) Ultraviolet \quad C) Infrared \quad D) X-rays

\textbf{Answer:} B \\
Ultraviolet radiation is the cause of sunburn.\\
The best answer is \textbf{B}.

\item \textbf{Question:} Which of the following is the most accurate measurement for the length of an object that is actually 15.0 m long?\\
A) 15.35 m \quad B) 15.2 m \quad C) 14.55 m \quad D) 14.5 m

\textbf{Answer:} \\
Step-by-step reasoning process: \textbf{(B)} \\
The uncertainty in the length of the object is $\pm 0.01$ m. The most accurate measurement is the one that is closest to the true length of the object. The true length of the object is 15.0 m, so the most accurate measurement is 15.2 m.\\
The best answer is \textbf{B}.

\item \textbf{Question:} Which type of graph would best display the changes in temperature over a 24 hour period?\\
A) line graph \quad B) pictograph \quad C) circle (pie) graph \quad D) stem-and-leaf graph

\textbf{Answer:} \\
Step-by-step reasoning process: \\
A line graph is used to show changes over time. A pictograph is used to show the number of people in different groups. A circle (pie) graph is used to show the number of people in different groups. A stem-and-leaf graph is used to show the number of people in different groups. Therefore, (A) is correct. (B)--(D) are incorrect because they are not used to show changes over time.\\
The best answer is \textbf{A}.

\item \textbf{Question:} Which is a true statement about cells?\\
A) Plant cells contain chloroplasts. \quad B) Animal cells are missing a nucleus. \quad C) Only plant cells have a cell membrane. \quad D) Animal cells include a rigid wall structure.

\textbf{Answer:} A \\
Plant cells contain chloroplasts, which are organelles that contain the pigment chlorophyll. This pigment is necessary for photosynthesis.\\
The best answer is \textbf{A}.

\item \textbf{Question:} \{question\}\\
A) \{option-a\} \quad B) \{option-b\} \quad C) \{option-c\} \quad D) \{option-d\} \\\\
\textbf{Answer}:

\end{enumerate}
\end{promptbox}

\begin{promptbox}[GPQA-Diamond with ARC-Easy in-context examples (Llama-3.2-1B-Instruct)]
\begin{enumerate}

\item 
\textbf{Question:} Which technology was developed most recently?\\
A) cellular telephone B) television C) refrigerator D) airplane

\textbf{Answer:} Step-by-step reasoning process: A- The cellular telephone was first developed in the 1970s.\\
The best answer is A.

\item 
\textbf{Question:} Which form of solar radiation causes sunburn?\\
A) Visible B) Ultraviolet C) Infrared D) X-rays

\textbf{Answer:} Step-by-step reasoning process: B -Ultraviolet radiation is the primary cause of sunburn.\\
The best answer is B.

\item 
\textbf{Question:} Oak trees produce seeds that are contained in acorns. Blue jays eat the seeds in acorns. Blue jays also collect acorns and hide them in the ground, often far away from the parent oak tree. Blue jays do not eat the seed of every acorn they hide. How do oak trees benefit from blue jays' collecting and hiding acorns?\\
A) The oak trees are pollinated by the blue jays. B) The oak trees are protected from other herbivores. C) The seeds of oak trees are protected from the sun. D) The seeds of oak trees are planted in new environments.

\textbf{Answer:} Step-by-step reasoning process: D -Blue jays collect acorns and hide them in the ground, often far away from the parent oak tree. This protects the seeds from other herbivores. Choice A is incorrect because the blue jays do not pollinate the oak trees. Choice B is incorrect because the blue jays do not protect the oak trees from other herbivores. Choice C is incorrect because the seeds of oak trees are not protected from the sun.\\
The best answer is D.

\item 
\textbf{Question:} The changing appearances of the nighttime sky over the surface of Earth and eclipses of the Moon have provided evidence that\\
A) Earth is a sphere. B) Earth supports life. C) Earth has a layered atmosphere. D) Earth is covered mostly with water.

\textbf{Answer:} A The changing appearances of the nighttime sky over the surface of Earth and eclipses of the Moon have provided evidence that Earth is a sphere.\\
The best answer is A.

\item 
\textbf{Question:} If the environment in an area is destroyed, a new community of organisms sometimes takes the place of previous communities. The introduction of a new community of organisms is known as\\
A) evolution. B) adaptation. C) biological diversity. D) ecological succession.

\textbf{Answer:} D -Ecological succession is the process of change in the species composition of a community over time.\\
The best answer is D.

\item 
\textbf{Question:} Which of these is a characteristic of the metal gold?\\
A) malleable B) lighter than water C) magnetic D) harder than a diamond

\textbf{Answer:} Step-by-step reasoning process: A -Gold is malleable, meaning it can be pounded into thin sheets. It is also denser than water, so it is heavier than water.\\
The best answer is A.

\item 
\textbf{Question:} When placed in a glucose solution, yeast cells will produce bubbles. Which process are the yeast cells carrying out?\\
A) chemosynthesis B) photosynthesis C) respiration D) transpiration

\textbf{Answer:} Step-by-step reasoning process: C -Yeast cells produce carbon dioxide gas as a byproduct of their metabolism.\\
The best answer is C.

\item 
\textbf{Question:}  \{question\}\\
A) \{option-a\} \quad B) \{option-b\} \quad C) \{option-c\} \quad D) \{option-d\} \\\\
\textbf{Answer}:

\end{enumerate}
\end{promptbox}

\subsubsection{HumanEval and LiveCodeBench Prompts for Code Generation}\label{appdix:prompts:humaneval_lcb}
\begin{promptbox}[HumanEval prompt template]
Solve the following code problem:\\
\{prompt\}\\

Complete the function below.\\
def \quad \{entry\_point\}

\end{promptbox}

\begin{promptbox}[LivdeCodeBench with HumanEval in-context examples (Llama-3.1-8B-Instruct)]

You are an expert Python programmer. You will be given a question (problem specification) and will generate a correct Python program that matches the specification and passes all tests.\\

\#\#\# Examples:
\begin{Verbatim}[breaklines,breakanywhere]
def has_close_elements(numbers: List[float], threshold: float) -> bool:
    """ Check if in given list of numbers, are any two numbers closer to each other than
    given threshold.
    >>> has_close_elements([1.0, 2.0, 3.0], 0.5)
    False
    >>> has_close_elements([1.0, 2.8, 3.0, 4.0, 5.0, 2.0], 0.3)
    True
    """
    for i in range(len(numbers)):
        for j in range(i + 1, len(numbers)):
            if abs(numbers[i] - numbers[j]) < threshold:
                return True
    return False

def separate_paren_groups(paren_string: str) -> List[str]:
    # Remove spaces from the input string
    paren_string = paren_string.replace(" ", "")

    # Initialize an empty list to store the separated groups
    separated_groups = []

    # Initialize an empty string to build the current group
    current_group = ""

    # Initialize a counter to keep track of the balance of parentheses
    balance = 0

    # Iterate over each character in the input string
    for char in paren_string:
        # If the character is an open parenthesis, increment the balance
        if char == "(":
            balance += 1
            current_group += char
        # If the character is a close parenthesis, decrement the balance
        elif char == ")":
            balance -= 1
            current_group += char
        # If the balance is zero, it means we've found a complete group
        if balance == 0:
            # Add the current group to the list of separated groups
            separated_groups.append(current_group)
            # Reset the current group and balance
            current_group = ""
            balance = 0

    # Return the list of separated groups
    return separated_groups

def truncate_number(number: float) -> float:
    """ Given a positive floating point number, it can be decomposed into
    and integer part (largest integer smaller than given number) and decimals
    (leftover part always smaller than 1).

    Return the decimal part of the number.
    >>> truncate_number(3.5)
    0.5
    """
    # First, we need to separate the integer part from the decimal part.
    # We can do this by using the built-in int() function to get the integer part.
    integer_part = int(number)
    
    # Now, we need to get the decimal part. We can do this by subtracting the integer part from the original number.
    # This will give us the decimal part, which is always smaller than 1.
    decimal_part = number - integer_part
    
    # Finally, we return the decimal part.
    return decimal_part

def below_zero(operations: List[int]) -> bool:
    """ You're given a list of deposit and withdrawal operations on a bank account that starts with
    zero balance. Your task is to detect if at any point the balance of account falls below zero, and
    at that point function should return True. Otherwise it should return False.
    >>> below_zero([1, 2, 3])
    False
    >>> below_zero([1, 2, -4, 5])
    True
    """
    balance = 0
    for operation in operations:
        balance += operation
        if balance < 0:
            return True
    return False

def mean_absolute_deviation(numbers: List[float]) -> float:
    """ For a given list of input numbers, calculate Mean Absolute Deviation
    around the mean of this dataset.
    Mean Absolute Deviation is the average absolute difference between each
    element and a centerpoint (mean in this case):
    MAD = average | x - x_mean |
    >>> mean_absolute_deviation([1.0, 2.0, 3.0, 4.0])
    1.0
    """
    # Calculate the mean of the input list
    mean = sum(numbers) / len(numbers)
    
    # Calculate the absolute difference between each element and the mean
    absolute_diffs = [abs(num - mean) for num in numbers]
    
    # Calculate the mean of the absolute differences
    mad = sum(absolute_diffs) / len(absolute_diffs)
    
    return mad
\end{Verbatim}

\#\#\# Question: \{question\}\\

\#\#\# Format: Implement a function called \verb|main()| which orchestrates the solution by reading inputs from stdin as \verb|sys.stdin.read().split()|, do not use \verb|input()|, and writing the answer to stdout. Feel free to use additional functions as necessary. Next do NOT forget to call \verb|main| function at the end of the program otherwise you will not be awarded any points.
\begin{Verbatim}
```python
# YOUR CODE HERE
```
\end{Verbatim}

\#\#\# Answer: (use the provided format with backticks)

\end{promptbox}

\begin{promptbox}[LivdeCodeBench with HumanEval in-context examples (Llama-3.2-3B-Instruct and Llama-3.2-1B-Instruct)]
You are an expert Python programmer. You will be given a question (problem specification) and will generate a correct Python program that matches the specification and passes all tests.\\

\#\#\# Examples:
\begin{Verbatim}[breaklines,breakanywhere]
def has_close_elements(numbers: List[float], threshold: float) -> bool:
    """ Check if in given list of numbers, are any two numbers closer to each other than
    given threshold.
    >>> has_close_elements([1.0, 2.0, 3.0], 0.5)
    False
    >>> has_close_elements([1.0, 2.8, 3.0, 4.0, 5.0, 2.0], 0.3)
    True
    """

    # Sort the list of numbers in ascending order
    numbers.sort()

    # Iterate over the sorted list of numbers
    for i in range(len(numbers) - 1):
        # For each number, check if the difference between it and the next number is less than the threshold
        if numbers[i + 1] - numbers[i] < threshold:
            # If the difference is less than the threshold, return True
            return True

    # If the function hasn't returned True after iterating over the entire list, return False
    return False

def truncate_number(number: float) -> float:
    """ Given a positive floating point number, it can be decomposed into
    and integer part (largest integer smaller than given number) and decimals
    (leftover part always smaller than 1).

    Return the decimal part of the number.
    >>> truncate_number(3.5)
    0.5
    """

    # First, we need to find the integer part of the number, which is the largest
    # integer smaller than the given number. We can use the math.floor function
    # for this.
    import math
    integer_part = math.floor(number)

    # Now, we need to find the decimal part of the number. This is simply the
    # difference between the original number and the integer part.
    decimal_part = number - integer_part

    # Finally, we return the decimal part of the number.
    return decimal_part

def below_zero(operations: List[int]) -> bool:
    """ You're given a list of deposit and withdrawal operations on a bank account that starts with
    zero balance. Your task is to detect if at any point the balance of account fallls below zero, and
    at that point function should return True. Otherwise it should return False.
    >>> below_zero([1, 2, 3])
    False
    >>> below_zero([1, 2, -4, 5])
    True
    """

    balance = 0
    for operation in operations:
        balance += operation
        if balance < 0:
            return True
    return False

def mean_absolute_deviation(numbers: List[float]) -> float:
    """ For a given list of input numbers, calculate Mean Absolute Deviation
    around the mean of this dataset.
    Mean Absolute Deviation is the average absolute difference between each
    element and a centerpoint (mean in this case):
    MAD = average | x - x_mean |
    >>> mean_absolute_deviation([1.0, 2.0, 3.0, 4.0])
    1.0
    """

    # Calculate the mean of the input list
    mean = sum(numbers) / len(numbers)
    
    # Initialize a variable to store the sum of absolute deviations
    sum_absolute_deviation = 0
    
    # Iterate over each number in the input list
    for num in numbers:
        # Calculate the absolute deviation of the current number from the mean
        absolute_deviation = abs(num - mean)
        
        # Add the absolute deviation to the sum
        sum_absolute_deviation += absolute_deviation
    
    # Calculate the mean absolute deviation by dividing the sum by the length of the list
    mad = sum_absolute_deviation / len(numbers)
    
    # Return the mean absolute deviation
    return mad

def filter_by_substring(strings: List[str], substring: str) -> List[str]:
    """ Filter an input list of strings only for ones that contain given substring
    >>> filter_by_substring([], 'a')
    []
    >>> filter_by_substring(['abc', 'bacd', 'cde', 'array'], 'a')
    ['abc', 'bacd', 'array']
    """

    # Initialize an empty list to store the filtered strings
    filtered_strings = []

    # Iterate over each string in the input list
    for string in strings:
        # Check if the substring is present in the current string
        if substring in string:
            # If the substring is found, append the string to the filtered list
            filtered_strings.append(string)

    # Return the filtered list of strings
    return filtered_strings
\end{Verbatim}

\#\#\# Question: \{question\}\\

\#\#\# Format: Implement a function called \verb|main()| which orchestrates the solution by reading inputs from stdin as \verb|sys.stdin.read().split()|, do not use \verb|input()|, and writing the answer to stdout. Feel free to use additional functions as necessary. Next do NOT forget to call \verb|main| function at the end of the program otherwise you will not be awarded any points.
\begin{Verbatim}
```python
# YOUR CODE HERE
```
\end{Verbatim}

\#\#\# Answer: (use the provided format with backticks)
    
\end{promptbox}

\end{appendices}


\end{document}